%% file: main.tex
\documentclass[11pt]{article}

\usepackage[T1]{fontenc}
\usepackage{lmodern}

\usepackage{amsmath}
\usepackage{amssymb}
\usepackage{mathtools}

\usepackage{graphicx}
\usepackage{float}
\usepackage{subcaption}

\usepackage{booktabs}
\usepackage{multirow}
\usepackage{array}

\usepackage{amsthm}

\newtheorem{theorem}{Theorem}
\newtheorem{lemma}{Lemma}
\newtheorem{proposition}{Proposition}
\theoremstyle{remark}
\newtheorem{remark}{Remark}
\newtheorem{definition}{Definition}

\usepackage{algorithm}
\usepackage{algpseudocode}

\usepackage{microtype}

\usepackage{booktabs}
\usepackage{tabularx}
\usepackage{array}

\usepackage{xcolor}
\usepackage{enumitem}

\usepackage[numbers,sort&compress]{natbib}

\usepackage{hyperref}
\usepackage[nameinlink,capitalize]{cleveref}

\hypersetup{
    colorlinks=true,
    linkcolor=blue,
    citecolor=blue,
    urlcolor=blue
}

\usepackage[
    margin=1in
]{geometry}

\newenvironment{IEEEkeywords}{\section*{Keywords}}{}

\title{
PAC-Private Autoregressive Generation: Calibrating Noise to Ensemble Disagreement
}

\author{
Mina Mirzadehsarcheshmeh,
Amir Keyvan Khandani
}

\date{}

\begin{document}

\maketitle

\begin{abstract}
Language models adapted on private text are often served through APIs, exposing information through generated tokens rather than model weights. Differentially private training protects the model parameters, while private prediction instead protects the released outputs. Existing private-prediction methods incur privacy cost with every release. In PMixED, maintaining a fixed privacy target increasingly shifts the output toward the public model as the generation horizon grows. PAC privacy offers a different approach. It calibrates noise to the variability of the output across possible secrets, so stable predictions require less noise. PAC-private prediction has been studied for classification, but not, to our knowledge, for autoregressive generation.

We construct $m=128$ overlapping worlds from the private corpus, with each record appearing in exactly $m/2$ worlds, and train one adapter per world over a frozen public model. The realized world is the secret. At each token, the public model defines a candidate set and the worlds vote within it. Their posterior-weighted disagreement determines the PAC noise, while unanimous positions require no calibration noise. We prove $I(S;Y_{1:T}) \le I(S;H_T) \le bT$ over $T$ released tokens. Our main contributions are extending PAC privacy to autoregressive generation, handling adaptive self-generated contexts, and introducing coupled decoding to preserve privacy accounting while avoiding greedy degeneration.

On WikiText-103 with GPT-2-small, we retain 74\% of the fine-tuning gain at a per-token
budget of $2^{-32}$, with membership-inference success bounded to at most 51.08\% after
$10^6$ tokens; leakage estimated directly from posterior entropy is roughly 17\% of the
charged budget. Inference privacy is not content protection: where membership advantage on a
memorized canary is indistinguishable from zero, the canary is still emitted at the same
rate. Against PMixED at matched membership-inference bounds on the same data universe and
test set, we retain 98\% of non-private headroom across horizons from $10^2$ to $10^6$
tokens, versus at most 56\%, with no crossover.\footnotemark[1]
\end{abstract}

\begin{IEEEkeywords}
PAC privacy, private prediction, autoregressive generation, large language models,
mutual information, adaptive composition, membership inference, coupled decoding,
memorization
\end{IEEEkeywords}

\footnotetext[1]{ChatGPT and Claude were used only for editorial assistance, including language polishing and clarity improvements.}


\input{sections/01_introduction}

\input{sections/02_background}

\input{sections/03_mechanism}

\input{sections/04_theory}

\input{sections/05_experimental_setup}

\input{sections/06_results}

\input{sections/07_discussion}

\input{sections/08_conclusion}


\bibliographystyle{plainnat}
\bibliography{references}


\clearpage
\appendix

\input{appendix/appendix}

\end{document}

%% file: sections/01_introduction.tex
\section{Introduction}
\label{sec:intro}

Language models adapted on private text can disclose that text through the interface they
are served on. A model fine-tuned on clinical notes, internal correspondence, or customer
records is typically deployed behind an API: users submit prompts and read generated
continuations, and never see the weights. Yet the continuations themselves carry
information about the adaptation corpus. Models memorize training sequences measurably
\citep{carlini2019secretsharer}, extraction attacks recover them from deployed systems
\citep{carlini2021extracting}, and membership inference can determine whether a given
record was used at all \citep{shokri2017mia,carlini2022lira}. The privacy question for a
served generative model is therefore not what its parameters reveal, but what its output
stream reveals over the course of a long interaction.

The standard answer is to privatize training. DP-SGD clips per-example gradients and adds
calibrated noise, yielding weights that satisfy differential privacy
\citep{abadi2016dpsgd}, after which serving is post-processing and costs nothing further.
The difficulty is that this method must remain private even against an adversary who holds the model parameters. In deployment, however, an adversary typically sees only the outputs. This stronger threat model comes at substantial cost because sensitivity is hard to compute tightly for non-convex training and is therefore enforced through clipping \citep{flemings2024,zsd}. The
alternative is to privatize the outputs directly. Private prediction
\citep{dwork2018prediction} answers queries through an interface that is itself private,
which suits an API deployment exactly, and recent systems \citep{ginart2022submix,flemings2024} bring it to next-token
prediction. However, privatizing outputs means the privacy budget
is consumed as predictions are served, and a generated passage is not one prediction but
hundreds. In PMixED, a fixed total privacy budget must be shared across the query horizon \citep{flemings2024}. As the horizon grows, the per-token constraint becomes tighter. This pushes the output toward the public model and reduces the contribution of the private experts, causing utility to decline.

PAC privacy \citep{xiao2023pac} offers a different trade. Rather than calibrating noise to
a worst-case sensitivity, it calibrates to the measured variability of the output across
the possible realizations of the secret, and bounds an adversary's posterior advantage at
any inference task through the mutual information between the secret and the release.
Stable outputs are therefore inexpensive to privatize, and the required noise can be estimated directly from the observed disagreement rather than from a conservative a priori bound. \citet{zsd} make this workable under adaptive,
adversarially chosen queries against a secret that persists, by having the curator track
the same posterior the adversary would and calibrate against it, and they instantiate it
for classification. They identify language-model inference as a promising future direction, especially for generating longer sequences while maintaining high utility. This paper
takes it up.

Doing so is not a matter of running the classification mechanism token by token. The
secret is drawn once and answers every token of the interaction, so the composition must
hold for a persistent secret rather than a resampled one. The query stream is
self-generated: each released token is appended to the context that produces the next
query, so adaptivity is built into the setting rather than merely permitted to an
adversary. The output space is a vocabulary of tens of thousands of tokens, over which
calibrating noise is computationally burdensome and highly damaging to utility. And
usable long-form text requires sampled decoding, yet letting each candidate model sample
independently would inject decoder randomness into the very quantity the calibration reads
--- manufacturing disagreement where the models actually agree, and inflating the noise
accordingly.

Our mechanism addresses these shortcomings together. The private corpus is assigned to $m = 128$
overlapping subsets, or \emph{worlds}, each record appearing in exactly $m/2$ of them, and
one adapter is trained per world on a frozen public base; the secret is which world is
real. At each token, the public model alone fixes a candidate set, every world votes for
one token within it, and the posterior-weighted disagreement among those votes determines
the calibrated noise --- so a position where all worlds agree carries no information about
which world is real and requires no noise at all. The curator adds that noise to the
realized world's vote, uses the resulting internal vector to update its posterior over
worlds, and releases only the argmax as a token. Sampling is preserved by sharing the same public randomness across all worlds. Each world still samples from its own temperature-$\tau$ conditional distribution. Because the worlds use the same random coins, worlds with similar probabilities often choose the same token.

This approach guarantees that the mutual information between the secret and the entire released
token stream is at most $bT$ for a per-token budget $b$ and $T$ released tokens, which
results in a bound on the success of any inference attack given only its prior. Four
issues require explicit treatment in our autoregressive instantiation. The privacy accounting is defined on internal vectors, but the user only sees released tokens. We therefore show that the token output leaks no more information than the internal mechanism. Under coupled decoding, this proof must also include the shared public randomness, because the candidate set depends on those random coins. One-hot votes
make the output covariance singular at every step and identically zero under unanimity, so
calibration must be defined on the active subspace and the belief update must use a
pseudoinverse. Autoregressive queries must be shown to satisfy the conditional
independence the composition theorem assumes, despite being functions of the mechanism's
own outputs. And coupled sampling must be shown to preserve the determinism the noise
calibration requires, which it does conditionally on coins that are public, independent of
the secret, and recorded in the history.

Empirically, on WikiText-103 with GPT-2-small, the mechanism retains at least 74\% of the
fine-tuning gain over the public model at a per-token budget of $2^{-32}$ --- a budget
that bounds membership-inference success to at most 51.08\% after $10^6$ released tokens ---
and accuracy is flat across the entire budget range from $2^{-4}$ to $2^{-32}$. That
flatness is not the budget failing to reach the mechanism: the calibrated noise scale
spans four orders of magnitude and flips on disagreeing positions rise from 26\% to 52\%,
but most positions are unanimous and therefore free of noise, and the additional flips are
close to accuracy-neutral. Because the $m$-world construction keeps the secret space
enumerable, we can also estimate the leakage directly from posterior entropy rather than
probing it with an attack, and find it to be roughly 17\% of what the fixed accounting
charges. Against PMixED at matched membership-inference bounds on identical data, our
mechanism retains 98\% of its non-private headroom at every horizon from $10^2$ to $10^6$
tokens, while PMixED retains at most 56\%, falling to between 1\% and 3\% by $10^6$ depending on how its budget is read, with no crossover.

Our contributions are as follows. To our knowledge, we give the first instantiation of
PAC-private prediction for autoregressive generation, together with the analysis the setting requires:
an end-to-end guarantee on the released token stream, and the four supporting results just
described. We also measure the actual leakage. It is about 17\% of the charged privacy budget and remains nearly constant across a wide range of budgets and generation lengths. This suggests that fixed per-token accounting is roughly six times more conservative than the measured leakage. The separately measured dissent fraction is
numerically consistent with the explanation that unanimous positions are free, but the two
are measured on different protocols, since the session runs do not record per-token
unanimity.
We identify a measurement hazard specific to private generation --- degenerate decoding
makes an ensemble agree, so degeneracy silently understates the measured privacy cost ---
and show that coupling the worlds through shared randomness removes the degeneracy at no
change to the accounting, bringing repetition and diversity to within a fraction of a
percent of human text at 256 tokens. We separate inference privacy from content
protection, exhibiting a budget at which membership advantage on a forcibly memorized
canary is statistically indistinguishable from zero while the canary itself is still
emitted at essentially the same rate whether or not the realized world memorized it.
Finally, we compare against PMixED under matched provable bounds, matched data, and both
of its accounting conventions, and find no crossover at any query volume we evaluate.

%% file: sections/02_background.tex
\section{Background}
\label{sec:background}

\subsection{Threat Model and What Leaks}
\label{sec:threat-background}

The operational question behind every guarantee in this paper is whether an adversary who
sees a model's outputs can determine what was in its training data. The canonical
formalization is the \emph{membership inference} game \citep{shokri2017mia}: a record is
either included in the training set or not, the adversary observes the trained system, and
success is predicting that bit better than chance. Membership inference is a natural
target because it poses a minimal and directly measurable question about training-data
participation, and because it comes with a well-defined prior against which any
improvement can be quantified. Modern attacks calibrate per-example difficulty rather
than comparing raw confidences, which raises measured success substantially and makes
weakly-defended models look much worse than earlier evaluations suggested
\citep{carlini2022lira}. A more demanding objective on the same axis is reconstructing
training records outright \citep{balle2022reconstructing}. A related but distinct
black-box threat is model extraction, where the target is the model's own functionality
rather than its training data \citep{tramer2016stealing}.

A second and distinct failure mode is \emph{memorization}: a model reproducing training
text verbatim rather than merely revealing that it was trained on it. Language models
memorize measurably \citep{carlini2019secretsharer}, and extraction attacks recover
training sequences from deployed models \citep{carlini2021extracting}. The distinction
matters for how guarantees are read. A bound on membership inference constrains what an
adversary can \emph{infer} about the training set; it does not by itself prevent the model
from \emph{emitting} a memorized string, since a string can be produced for reasons that
have nothing to do with the private copy. We return to this separation in
Section~\ref{sec:canary-results}, where both are measured.

We assume throughout that the model is served behind an interface: the adversary submits
inputs and observes outputs, chosen adaptively, but does not see the weights. This is the
deployment mode of commercial language models \citep{flemings2024}, and, as the next
subsection argues, it is also where the standard privacy machinery is a poor fit.

\subsection{Differential Privacy and Its Limits for Language Models}
\label{sec:dp-background}
\paragraph{Definition and privacy accounting.}
Differential privacy~\cite{dwork2006dp,dwork2006calibrating} bounds the influence of any single record by requiring the output distribution to be nearly unchanged across adjacent datasets: a mechanism is $(\varepsilon,\delta)$-DP if, for all adjacent datasets and all measurable sets of outputs, the probability under one dataset is at most a factor $e^\varepsilon$ larger than under the other, plus $\delta$. The guarantee is worst-case over datasets, holds regardless of the adversary's prior information, and degrades gracefully under composition, which is what makes it usable for iterative algorithms~\cite{dwork2014foundations}.

Because direct composition of $(\varepsilon,\delta)$-DP guarantees can yield loose bounds, practical accounting is often performed using alternative privacy representations with simpler or tighter composition rules. R\'enyi differential privacy tracks a R\'enyi divergence of order $\alpha$~\cite{mironov2017rdp}, zero-concentrated DP controls privacy loss through R\'enyi-divergence bounds that imply sub-Gaussian concentration~\cite{bun2016zcdp}, and $f$-DP characterizes privacy through a hypothesis-testing trade-off function, while Gaussian DP specializes this framework to the trade-off between shifted Gaussian distributions~\cite{dong2022gdp}. These are privacy definitions in their own right, each with its own semantics, but in machine learning they are often used for privacy accounting because their composition rules are simpler or tighter; when an $(\varepsilon,\delta)$-DP guarantee is required, the accumulated guarantee can then be converted at the end. The choice is not neutral in practice: that conversion is where a method's reported budget is fixed, and different readings of the same accountant can differ substantially. This matters directly in Section~6.5, where the baseline we compare against spends its budget in RDP and we report the comparison under two conversion conventions.

\paragraph{Private training, and three obstacles.}
The dominant way to obtain a DP model is to privatize training: clip per-example gradients
to bound each sample's contribution to the update, and add calibrated noise at each step
\citep{song2013sgd,abadi2016dpsgd}. Once training is complete, the released weights carry a
DP guarantee, and any subsequent inference is post-processing. Applied to language models,
three difficulties arise.

First, privacy calibration in DP-SGD requires controlling \emph{sensitivity}, the worst-case
change induced by replacing one training record. Computing sensitivity tightly is NP-hard in
general, and tight sensitivity bounds remain unavailable for generic non-convex optimization
\citep{xiao2023pac}. DP-SGD therefore \emph{imposes} a sensitivity bound by clipping
per-example gradients rather than deriving a tight bound for the training procedure
\citep{abadi2016dpsgd,xiao2023pac}. Instance-specific refinements such as smooth sensitivity
exist \citep{nissim2007smooth}, but they do not by themselves resolve the sensitivity analysis
of iterative non-convex training procedures.

Second, the mechanics are expensive at scale: per-example gradients raise memory and time
costs well above ordinary training \citep{flemings2024}. Third, and most consequential for
our setting, DP-SGD privatizes \emph{weights}, which means it defends against an adversary
holding the parameters --- while a deployed language model exposes only its outputs.
Defending against white-box access that the deployment does not grant means paying for a
threat model stronger than the one that applies \citep{flemings2024,zsd}.
Parameter-efficient private fine-tuning reduces the computational and memory cost of DP
training \citep{yu2022dpfinetune}, but it does not remove the underlying threat-model
mismatch.

\subsection{Private Prediction}
\label{sec:private-prediction-background}

\paragraph{The paradigm.}
If only outputs are exposed, privacy can instead be enforced at the prediction
interface. \citet{dwork2018prediction} formalized this setting as
\emph{private prediction}: rather than releasing a model trained under DP, users
access the model only through an interface whose predictions satisfy differential
privacy. For learning problems with sufficiently stable predictions, privacy can
then be calibrated to the sensitivity of the predictions themselves, potentially
avoiding some of the overhead associated with releasing a fully private model.

The corresponding cost is a change in how the budget is spent. For a DP-trained model the
guarantee is discharged once, at training time, and serving predictions from the released
weights is post-processing that consumes nothing further. For output-private prediction it
is the releases themselves that are privatized, so budget is consumed as predictions are
answered, and the guarantee bounds how many a deployment can serve --- the trade-off
analysed by \citet{vandermaaten2020tradeoffs}. This is what makes composition the central
question for our setting rather than a secondary concern. A generated passage is not one
release but hundreds: whatever is spent per token is multiplied by the number of tokens
served, and each released token becomes part of the context for the next query, so the
relevant quantity is how leakage accumulates over a long and inherently adaptive
interaction.

\paragraph{Ensembles and noisy aggregation.}
The dominant construction trains an ensemble on disjoint shards of the private data and
releases a noisily aggregated vote, so that no single record can move the answer much.
PATE \citep{papernot2017pate} introduced this for classification and transfers the
ensemble's knowledge to a student model trained on public data. Its refinement
\citep{papernot2018scalable} adds a Gaussian aggregation mechanism together with a
data-dependent analysis under which strong teacher consensus yields a substantially lower
privacy cost, and a confident aggregator that declines to answer queries whose vote is not
decisive; both are what make the approach practical beyond small benchmarks. The broader
pattern --- that agreement among models is what makes a release cheap --- recurs in this
literature, and in a different form in ours.

\paragraph{Language models.}
Representative systems for language-model private prediction include SubMix and PMixED.
SubMix \citep{ginart2022submix} fine-tunes an ensemble on disjoint parts of a private
corpus and mixes the ensemble's next-token distribution with that of a public pre-trained
model, with a mixing weight tuned to the degree of consensus among ensemble members: where
the models agree, the released distribution can safely follow them. It provides a tight,
data-dependent accountant under a relaxation of group DP for next-token prediction. PMixED
\citep{flemings2024} gives a per-token protocol with a standard DP guarantee: each
expert's output distribution is projected toward the public model within a bounded
divergence, the projected distributions are averaged, and the budget is tracked in RDP
with Poisson subsampling amplification and converted at the end. It is the closest DP
private-prediction baseline to our setting, and the one we evaluate against in
Section~\ref{sec:pmixed-results}.

PMixED's accounting has a structural consequence worth naming, because it governs that
comparison. The privacy of a release comes from moving the output toward the public model,
and the required amount of movement is set by the budget available per token. Under a
fixed total privacy target spread over a horizon of $T$ tokens, that per-token budget
shrinks as $T$ grows, the projection tightens, and the private experts contribute less.
The mechanism does not degrade because the models get worse; it degrades because the
accountant demands it. SubMix's data-dependent accounting spends according to realized
consensus rather than a horizon fixed in advance, so this particular characterization
should not be read across to it. AdaPMixED extends PMixED with data-dependent accounting and an adaptive
screening mechanism, reducing the realized privacy cost on suitable
queries \citep{flemings2024adaptive}. Unlike PMixED, its privacy cost
depends on the realized query sequence. Our matched comparison therefore
uses PMixED, whose privacy guarantee is fixed before deployment, while
we discuss AdaPMixED as a complementary adaptive alternative.

\paragraph{Other routes for language models.}
Private prediction is not the only option when the private data need not be trained on
directly. Where the private corpus can be used through the prompt rather than the weights,
one can generate synthetic demonstrations under DP and use those for in-context learning
\citep{tang2024icl}, which moves the guarantee to the data-generation step. These
approaches are complementary to ours: they privatize what enters the context, whereas we
privatize what leaves the decoder.

\subsection{PAC Privacy}
\label{sec:related-pac}

\paragraph{Instance-based privacy from simulation.}
PAC privacy \citep{xiao2023pac} departs from the indistinguishability tradition on two
counts. It asks a different question --- not how distinguishable $M(X_0)$ and $M(X_0')$
are for adjacent inputs, but how hard it is to \emph{reconstruct} $X$ from the release ---
and it answers it relative to an assumed input distribution rather than a worst case over
inputs. Formally, an adversary's posterior success at any inference criterion is bounded
through an $f$-divergence between the joint and product distributions of $(X, M(X))$,
which for the KL case reduces to the mutual information $I(X; M(X))$. Two consequences
drive everything downstream. First, the guarantee is attack-agnostic: a single
mutual-information budget constrains every inference task at once, given only its prior.
Second, and this is what makes the framework operational, the required noise is
determined by the \emph{covariance} of the output rather than by its worst-case
sensitivity, so it can be estimated by Monte Carlo simulation of a black-box mechanism
with a stated confidence. Where differential privacy needs a sensitivity bound that is
NP-hard in general and unknown for non-convex training, PAC privacy substitutes
simulation for analysis. The utility consequence is that noise scales with
$\sum_j \sqrt{\lambda_j}$ over the output covariance spectrum rather than with the
ambient dimension: a mechanism whose output concentrates in a low-rank subspace is cheap
to privatize, which is precisely the regime a stable prediction interface occupies. A
recent overview \citep{xiao2025blackbox} gives a compact statement of this position.

\paragraph{From definition to deployable mechanism.}
\citet{sridhar2025pacalgorithms} supply the missing engineering: an efficient simulation
procedure that determines \emph{anisotropic} noise for a target budget, applied to
K-Means, SVM, PCA and random forests, with the resulting guarantees checked against
empirical attacks. Two of their findings recur in our setting. Anisotropic calibration is
substantially better than isotropic noise when the output varies unevenly across
directions, and algorithmic stability --- including stability induced deliberately, by
regularization --- translates directly into privacy amplification. This is the same
mechanism by which cross-world agreement makes our per-token releases inexpensive.

\paragraph{Adaptivity, and why it is the hard part.}
Both of the above concern a single release, or a batch of releases fixed in advance. An
interactive service does not work that way: a user observes a response and chooses the
next query in light of it. Two distinct difficulties are folded into that observation, and
the lineage separates them.

\citet{sridhar2026pacdb} address the first. They identify adaptivity as the obstacle to
applying PAC privacy to databases --- prior results would require all queries to be
declared \emph{a priori}, which they call a fundamental usability limitation --- and prove
that the budget composes linearly across adaptively and even adversarially generated
queries, provided the input for each query is sampled \emph{independently}. Their
privatization layer realizes this by resampling the input for each output cell, so that
the query sequence may adapt while the secret behind each response does not persist.
This is the right model for a database that repeatedly draws fresh subsets, and it is not
the model of a served language model: our secret world is drawn once and answers every
token of the interaction.

For a genuinely persistent secret, composition was possible before \citet{zsd} but costly.
The posterior-oblivious result of \citet{xiao2024thesis} calibrates each step against the
prior, and, as \citet{zsd} analyse, its budget then either grows quadratically in the
number of releases or, if the per-step noise is set to the input-independent worst case,
retreats to a DP-like level that discards the instance-based advantage the framework
exists to exploit. This is the second difficulty, and it is ours.

\citet{zsd} resolve it. Their insight is that if the adversary tracks a posterior over the
secret, the curator must do so too: calibrating each release against the prior protects
only an adversary who has observed nothing. Their mechanism maintains the exact posterior
as a belief state, calibrates each step's noise to that belief, and updates it by Bayes'
rule, under which mutual information accumulates linearly in the number of releases even
though a single secret persists and the queries are chosen adaptively by an adversary.
They instantiate this for classification by restricting the secret space to $m$ subsets of
a data universe with each record in exactly $m/2$ of them, which keeps the covariance and
the belief update exactly computable while fixing the membership prior at $1/2$; they
report 87.79\% CIFAR-10 accuracy at a per-query budget of $2^{-32}$, bounding
membership-inference success below 51.08\% after $10^6$ queries, and distil private
predictions into a publishable student model. Theirs is the composition result our
mechanism uses directly, since an autoregressive interaction is precisely a long sequence
of adaptive queries against one persistent secret; Section~\ref{sec:prelim} states the
parts we rely on.

\paragraph{Scope of the framework.}
PAC privacy has been applied beyond classification: to generative image models through
private classifier guidance in the sampling process \citep{xu2024pacdiffusion}, and to
settings where the perturbation
itself is constrained, such as one-sided bounded noise for side-channel obfuscation
\citep{xiao2025onesided}. Related lines ask what can be proved about encodings that
preserve learnability \citep{xiao2024encoding} and how a PAC-privacy claim can be
verified by a third party without trusting the curator, via zero-knowledge proofs
\citep{repetto2026zkpac}. The diffusion work is a prior instance of PAC privacy applied to a generative model, though it privatizes attributes of sampled images rather than a sequential release interface, so the persistent-secret
adaptive release-composition problem that dominates our setting is not
studied there.

\paragraph{Known slack in the Gaussian surrogate.}
The automation of PAC privacy rests on a specific device: bounding mutual information by a
Gaussian expression in the output covariance, which is what makes the required noise
estimable from simulation. That expression is an upper bound on the mutual information,
not an identity. \citet{zhang2026rpac} characterize when it is tight --- only when the
unperturbed output is Gaussian and the added noise is independent Gaussian --- so for
other output distributions some budget is spent without being needed. They propose tighter
post-processing based on Donsker--Varadhan and sliced-Wasserstein representations, and
then a residual measure of the privacy remaining after adversarial inference, implemented
by selecting noise distributions through bilevel optimization. The observation concerns
the covariance-based Gaussian surrogate rather than the PAC privacy definition, which is
stated in terms of mutual information directly. We use the standard Gaussian calibration
throughout, so this slack is present in our mechanism too; it compounds with a second and
unrelated conservatism we quantify in Section~\ref{sec:privacy-results}, where the charged
budget exceeds the measured leakage by roughly a factor of six for reasons specific to
fixed-budget accounting.

\subsection{Preliminaries}
\label{sec:prelim}

This subsection collects the definitions and results from PAC privacy
\citep{xiao2023pac} and its private-prediction instantiation \citep{zsd} that the rest of
the paper uses. All mutual-information quantities are in nats.

\paragraph{Setup.}
A sensitive input $S$ is drawn from a distribution $P_S$ over a domain $\mathcal{S}$, and
a mechanism $M$ releases $M(S)$. An adversary observes the release and attempts an
inference task specified by a binary criterion $\rho(\hat{S}, S) \in \{0,1\}$, which
equals one when the adversary's estimate $\hat{S}$ counts as a success. Unlike
worst-case definitions, which quantify how distinguishable $M(S_0)$ and $M(S_0')$ are for
adjacent inputs, PAC privacy is stated directly in terms of how hard inference is under a
given input distribution.

\begin{definition}[PAC privacy \citep{xiao2023pac,zsd}]
\label{def:pac}
Let $S \sim P_S$ and let $M : \mathcal{S} \to \mathcal{R}$ be a possibly randomized
mechanism. $M$ is $(\delta, \rho, P_S)$-PAC private if for every informed adversary
$\mathcal{A}$ that knows $(P_S, M)$, observes $R = M(S)$, and outputs
$\hat{S} = \mathcal{A}(R)$,
\[
  1 - \delta_{\mathcal{A}} \;:=\; \Pr\big[\rho(\hat{S}, S) = 1\big] \;\le\; 1 - \delta ,
\]
the probability being over $S \sim P_S$, the randomness of $M$, and the adversary.
\end{definition}

The guarantee is relative to what the adversary could already do without the release.

\begin{definition}[Optimal prior success rate \citep{xiao2023pac}]
\label{def:prior}
For a criterion $\rho$ and distribution $P_S$, the optimal prior success rate is
$1 - \delta_0^{\rho} := \sup_{\hat{S}} \Pr_{S \sim P_S}[\rho(\hat{S}, S) = 1]$, the best
an adversary can do knowing only $P_S$.
\end{definition}

\paragraph{From mutual information to attack success.}
The central quantity is the mutual information between the secret and the release, which
bounds the adversary's improvement over the prior for \emph{every} criterion $\rho$
simultaneously.

\begin{theorem}[Posterior advantage bound \citep{xiao2023pac}]
\label{thm:mi-to-attack}
For any criterion $\rho$, the posterior success rate $1 - \delta_{\mathcal{A}}$ of an
informed adversary satisfies
\[
  D_{\mathrm{KL}}\big(\mathbf{1}_{\delta_{\mathcal{A}}} \,\|\,
  \mathbf{1}_{\delta_0^{\rho}}\big)
  \;\le\; I\big(S; M(S)\big),
\]
where $\mathbf{1}_{p}$ denotes a Bernoulli distribution with parameter $p$. In
particular, by Pinsker's inequality the posterior advantage obeys
$\delta_0^{\rho} - \delta_{\mathcal{A}} \le \sqrt{\tfrac{1}{2} I(S; M(S))}$.
\end{theorem}

Two consequences matter here. The bound is attack-agnostic: it depends only on $P_S$ and
$M$, so enforcing a mutual-information budget $I(S; M(S)) \le B$ constrains every
inference attack at once, given only its prior. And converting a budget into a concrete
guarantee requires just two numbers, the prior and the budget --- which is why the
guarantee tables of \citet{zsd} transfer to any instantiation with the same priors.

\paragraph{Calibrating noise to a budget.}
For deterministic mechanisms, Gaussian noise calibrated to the \emph{covariance} of the
output enforces a target budget.

\begin{theorem}[Noise determination \citep{xiao2023pac,zsd}]
\label{thm:noise}
Let $M : \mathcal{S} \to \mathbb{R}^d$ be deterministic, let
$\operatorname{Var}(M(S)) = U \Lambda U^{\top}$ with
$\Lambda = \operatorname{diag}(\lambda_1, \dots, \lambda_d)$, and let $B > 0$. Define
\[
  \Lambda_B \;=\; \operatorname{diag}\!\left(
  \frac{\sqrt{\lambda_i} \sum_{j=1}^{d} \sqrt{\lambda_j}}{2B} \;:\; i = 1, \dots, d
  \right).
\]
Then $I(S; M(S) + Z) \le B$ for $S \sim P_S$ and independent
$Z \sim \mathcal{N}(0, U \Lambda_B U^{\top})$.
\end{theorem}

The noise magnitude depends on $\sum_j \sqrt{\lambda_j}$ rather than on the output
dimension $d$: a mechanism whose output varies within a low-rank subspace needs noise
scaled to that rank, not to the ambient dimension. This is the structural difference from
differential privacy, where the required noise norm scales as $\Theta(\sqrt{d})$ at fixed
sensitivity \citep{xiao2023pac}, and it is what makes stable, low-rank outputs cheap to privatize. When
$\mathcal{S}$ is finite with $|\mathcal{S}| = m$, the covariance can be computed exactly
by evaluating $M$ on all $m$ inputs, at cost dominated by those $m$ evaluations
\citep{zsd}; when it is not, the covariance is estimated by Monte Carlo simulation, and
the budget then holds with a stated confidence \citep{xiao2023pac}.

It is convenient to name the map from a distribution, a mechanism, and a budget to a
noise covariance.

\begin{definition}[Noise calibration function \citep{zsd}]
\label{def:sigma}
$\Sigma : \mathcal{P}(\mathcal{S}) \times (\mathbb{R}^d)^{\mathcal{S}} \times
\mathbb{R}_{>0} \to \mathbb{R}^{d \times d}$ is a noise calibration function if for every
distribution $P$ over $\mathcal{S}$, deterministic $M$, and $B > 0$,
\[
  I_{S \sim P}\big(S;\, M(S) + \mathcal{N}(0, \Sigma(P, M, B))\big) \;\le\; B .
\]
Theorem~\ref{thm:noise} instantiates such a $\Sigma$.
\end{definition}

\paragraph{Composition under a persistent secret.}
Serving predictions means answering a stream of queries chosen by the user, possibly
adaptively and adversarially, against a secret that is drawn once and persists. Formally,
at step $t$ the adversary submits a mechanism $M_t : \mathcal{S} \to \mathbb{R}^{d_t}$
selected from the interaction history $H_{t-1} = \{(M_i, R_i)\}_{i<t}$ and its own
randomness, subject only to the absence of a side channel,
\begin{equation}
\label{eq:no-side-channel}
  M_t \perp S \mid H_{t-1},
\end{equation}
and the curator returns a noised release $R_t$. Calibrating each step's noise to the
prior $P_S$ is not sufficient here: it protects against an adversary who has observed
nothing, while a rational adversary tracks the posterior $P_{S \mid H_{t-1}}$. Prior
composition results for PAC privacy therefore either require non-adaptive queries, or
resample the secret at each step, or accumulate quadratically in $T$ \citep{zsd}.

\citet{zsd} resolve this by making the curator adapt as well. Their mechanism maintains a
belief state $P_t$, calibrates each release against the \emph{current} belief rather than
the prior, and updates the belief by Bayes' rule:
\begin{equation}
\label{eq:posterior-update}
\begin{aligned}
&\Sigma_t = \Sigma(P_{t-1},M_t,b_t),\\
&R_t = M_t(S)+Z_t,\qquad Z_t\sim\mathcal{N}(0,\Sigma_t),\\
&P_t(s)\propto P_{t-1}(s)\exp\!\left[
-\frac{1}{2}(R_t-M_t(s))^\top
\Sigma_t^{-1}(R_t-M_t(s))
\right].
\end{aligned}
\end{equation}

\begin{lemma}[Tracked belief is the true posterior \citep{zsd}]
\label{lem:tracked-posterior}
The belief state maintained by \eqref{eq:posterior-update} satisfies
$P_t = P_{S \mid H_t}$ at every step.
\end{lemma}

\begin{theorem}[Adversarial composition \citep{zsd}]
\label{thm:composition}
Under \eqref{eq:no-side-channel}, the releases produced by \eqref{eq:posterior-update}
with per-step budgets $\{b_t\}$ satisfy
$I(S; R_1, \dots, R_T) \le I(S; H_T) \le \sum_{t=1}^{T} b_t$.
\end{theorem}

Accumulation is linear rather than sublinear because mutual information is an expected KL
divergence and composes additively by the chain rule; the same is true of the parameters
of Rényi and concentrated DP, which appear sublinear only after conversion to
$(\varepsilon, \delta)$-DP \citep{zsd}. Theorem~\ref{thm:composition} equivalently bounds
$\mathbb{E}[D_{\mathrm{KL}}(P_{S \mid H_T} \,\|\, P_S)]$, the expected refinement of the
adversary's belief over the prior.

\paragraph{Instantiation: membership inference and the prior floor.}
The criterion $\rho$ specialises to membership inference by taking $P_S$ to be a
distribution over subsets of a finite universe $U$ and asking the adversary to predict
whether a given record lies in the realized subset \citep{xiao2023pac}. \citet{zsd}
make the secret space tractable by restricting it to $m$ subsets
$\mathcal{S} = \{S_1, \dots, S_m\}$ with $P_S$ uniform, assigning each record to exactly
$m/2$ of them; every record then has membership prior exactly $1/2$, matching Poisson
subsampling at rate $0.5$ on the marginal, while keeping exact covariance evaluation and
exact belief updates feasible. Two priors follow: $1/2$ for membership of any individual
record, and, by the following bound, $1/m$ for recovery of the secret itself.

\begin{proposition}[Prior floor \citep{zsd}]
\label{prop:prior-floor}
Let $P_S$ be uniform over $m$ subsets. For exact identification of the secret, the
optimal prior success rate is exactly $1 - \delta_0^{\rho} = 1/m$. More generally, for
any criterion under which correctly identifying the secret counts as a success, the
optimal prior success rate satisfies $1 - \delta_0^{\rho} \ge 1/m$.
\end{proposition}

Combining Theorem~\ref{thm:composition} with Theorem~\ref{thm:mi-to-attack} at these two
priors converts a per-step budget $b$ and a horizon $T$ into concrete bounds on
attack success. \citet{zsd} tabulate these, and since they are functions of
$(b, T, \text{prior})$ alone, they apply unchanged to any mechanism satisfying the same
composition guarantee at the same priors --- including ours, once $T$ is reinterpreted
as a count of released tokens (Section~\ref{sec:theory}).

%% file: sections/03_mechanism.tex
\section{PAC-Private Autoregressive Generation}
\label{sec:mechanism}

Our objective is to serve autoregressively generated text from a language model adapted
to a private corpus while placing a provable bound on the success of membership
inference by an informed adversary who observes the released token stream and knows the
mechanism, training procedure, and candidate worlds, but not which world is realized.
PAC privacy is well suited to this setting because it privatizes model outputs rather
than model weights and calibrates noise to how much the output changes across possible
secrets. The challenge in autoregressive generation is that this variation must be
evaluated online at every token, under a large vocabulary and a query stream that is
itself determined by previous releases.

We make the secret space finite and directly evaluable at serving time. The private
corpus is assigned to $m=128$ overlapping subsets, or \emph{worlds}, with each record
appearing in exactly $m/2=64$ worlds. One adapted model is trained per world, and the
secret is the index of the realized world. At each token, all worlds are evaluated on
the current context. Their votes specify how the mechanism would behave under each
possible secret, which determines the posterior-weighted vote covariance and therefore
the PAC-calibrated noise. When all worlds vote identically, the covariance is zero and
no calibration noise is required, although the fixed accounting still charges the
per-token budget $b$. To support sequential interaction against an adaptive adversary,
the curator tracks a posterior over worlds and recalibrates at every step.

\begin{figure}[H]
    \centering
    \makebox[\textwidth][c]{%
        \includegraphics[width=1.08\textwidth]{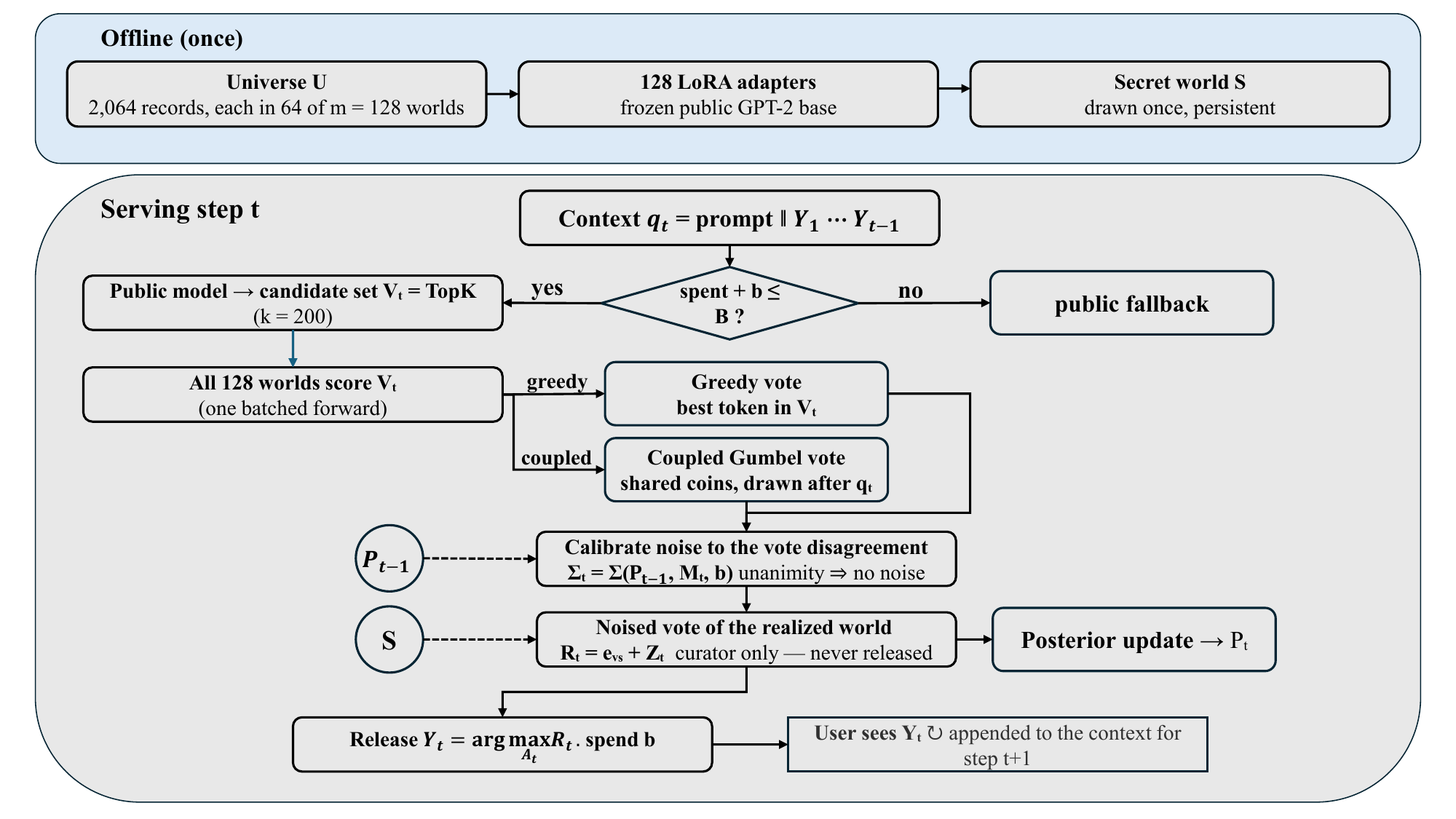}
    }
    \caption{\textbf{PAC-private autoregressive generation.}
    Offline, the private universe is assigned to $m=128$ overlapping worlds,
    with each record appearing in exactly 64 worlds, and one LoRA adapter is
    trained per world on a frozen public GPT-2 base. At serving time  while the privacy budget permits, the
    public model defines the candidate set $V_t$, and all 128 worlds vote
    within this set using either greedy decoding or shared-randomness
    Gumbel sampling. Dashed arrows mark the two quantities entering a serving step from
    outside it: the current posterior $P_{t-1}$ and the realized secret $S$.
    Their posterior-weighted disagreement determines the
    PAC-calibrated noise. The curator adds this noise to the realized
    world's vote to form the internal vector $R_t$, updates the posterior
    over worlds using $R_t$, and releases only
    $Y_t=\arg\max_{v\in A_t}(R_t)_v$. The released token is appended to the
    context for the next autoregressive step.}
    \label{fig:mechanism}
\end{figure}

Moving from classification to autoregressive generation introduces four additional
issues: the 50,257-token output space, autoregressive feedback from released tokens into
future queries, the fact that accounting is defined on an internal noised vector while
the user receives only a discrete token, and the need for sampled decoding without
creating artificial disagreement between otherwise similar worlds. The mechanism below
addresses these with public top-$k$ support, posterior-aware sequential calibration,
token-only post-processing, and shared-randomness Gumbel decoding, respectively.
Figure~\ref{fig:mechanism} summarizes the complete pipeline.

Throughout, $S\sim P_S$ denotes the secret, $M_t$ the query mechanism at step $t$,
$R_t$ the internal vector release, $H_t=\{(M_i,U_i,R_i))\}_{i=1}^t$ the interaction
history (where $U_i$ is absent under greedy decoding), $\Sigma(\cdot,\cdot,\cdot)$ the PAC noise-calibration function, and $b$ the
per-step mutual-information budget in nats.

\subsection{Threat Model and Secret}
\label{sec:threat-model}

Let $U=\{u_1,\ldots,u_n\}$ be a private corpus of $n=2{,}064$ records: 2,048
WikiText-103 articles and 16 canary records used later for memorization tests. Following
the $m/2$ subsampling construction, we form $m=128$ worlds
$\mathcal{S}=\{S_1,\ldots,S_m\}$ by assigning every record independently to exactly
$m/2=64$ worlds chosen uniformly at random. The secret is
\[
S\sim P_S=\mathrm{Uniform}(\mathcal{S}),
\]
so every record has membership prior
\[
\Pr[u\in S]=\frac{1}{2}.
\]
Under the uniform world prior, exact world identification has prior success
$1/m=1/128$; PAC privacy bounds how much this success can increase after observing the
released interaction.

The adversary is informed and adaptive: it knows $U$, $P_S$, the training pipeline,
all $m$ trained models, and the mechanism. Its uncertainty is only the realized world
index. It may choose each query as a function of the previously observed history,
subject to the no-side-channel condition
\[
M_t \perp S \mid H_{t-1}.
\]
The secret is persistent: one world is drawn once and remains the realized world for
the entire interaction. Fresh public randomness used at step $t$ is generated only
after the current query has been fixed.

Record membership is a derived binary quantity,
$\mathbf{1}[u\in S]$, obtained by summing posterior mass over the 64 worlds containing
record $u$. This is coarser than identifying the exact world; the empirical behavior of
record-level MIA and world identification is evaluated separately in the results.

The balanced $64/128$ assignment is also useful for controlled memorization tests. Under
the initial uniform posterior, if a record were memorized perfectly by all and only its
member worlds, the induced binary vote split would carry equal prior mass on the two
membership hypotheses. Section~\ref{sec:canary-results} tests this case explicitly using
controlled canary memorization.

\subsection{Worlds as LoRA Adapters}
\label{sec:world-construction}

Training 128 full language models from scratch is impractical. Instead, each world is a
LoRA adapter on a frozen public GPT-2-small base. All adapters are trained with the same
pipeline on their assigned subsets; the detailed architecture and hyperparameters are
reported in Section~\ref{sec:setup}. We write
$p_s(\cdot\mid q)$ for the next-token distribution produced by the adapter associated
with world $s$ on context $q$, and $p_{\mathrm{pub}}$ for the public base model.

This shared-base construction makes the 128 worlds inexpensive enough to evaluate
jointly at serving time. The resulting cross-world agreement and memorization behavior
are measured empirically rather than assumed.

\subsection{Per-Token Mechanism}
\label{sec:per-token-mechanism}

At step $t$, the context is
\[
q_t=(\text{prompt}\,\|\,Y_1\cdots Y_{t-1}),
\]
so each released token becomes part of the next query.

\paragraph{Public support.}
The public model defines the admissible candidate set
\[
V_t=\mathrm{TopK}\!\left(p_{\mathrm{pub}}(\cdot\mid q_t),k\right),
\qquad k=200.
\]
The public model only defines this set; it does not supply a fallback vote. Each private
world chooses its own token within $V_t$. If a world's unrestricted top-1 token lies
outside $V_t$, that world instead votes for its highest-probability token inside
$V_t$. Since all subsequent voting and release occur inside $V_t$, no token outside
$V_t$ can be emitted. This is an \emph{off-support non-emission guarantee}, not a
general content-privacy guarantee.

\paragraph{Votes and active subspace.}
Under greedy decoding, world $i$ votes
\[
v_i=\arg\max_{v\in V_t}p_{S_i}(v\mid q_t),
\]
and the corresponding output is the one-hot vector $y_i=e_{v_i}$. The step-$t$ mechanism is therefore the vote map
$M_t : s \mapsto e_{v(s)}$, and $\Sigma$ is applied to it in the sense of Definition~3. Under coupled
Gumbel decoding, the vote is instead sampled as described in
Section~\ref{sec:gumbel}.

Let $A_t$ be the set of distinct votes and let $a_t=|A_t|$. Define the posterior vote
mass
\[
\pi_t(v)=\sum_{i:v_i=v}P_{t-1}(S_i).
\]
The posterior-weighted covariance of the one-hot vote is then
\[
C_t=\operatorname{diag}(\pi_t)-\pi_t\pi_t^\top.
\]
Coordinates outside $A_t$ have zero variance. Because one-hot vectors sum to one,
$C_t$ has a null direction, so the active calibration subspace has dimension at most
$a_t-1$. The small values of $a_t$ observed in practice are reported in the stability
results; they are empirical properties, not structural guarantees.

\paragraph{Calibration and release.}
Following the PAC noise-calibration result of \citet{zsd}, the noise covariance is calibrated from the current posterior, query mechanism, and
per-step budget,
\[
\Sigma_t=\Sigma(P_{t-1},M_t,b),
\]
using the active subspace of $C_t$. The curator then forms the internal vector
\[
R_t=e_{v_S}+Z_t,
\qquad
Z_t\sim\mathcal{N}(0,\Sigma_t),
\]
where $v_S$ is the vote of the realized world. Only the token
\[
Y_t=\arg\max_{v\in A_t}(R_t)_v
\]
is released. If all worlds are unanimous, then $C_t=0$, no calibration noise is
required, and the release is deterministic (Lemma~\ref{lem:singular}).

\paragraph{Posterior update and privacy-budget tracking.}
The curator updates its posterior with the internal vector $R_t$ under its exact Gaussian
likelihood, as in Algorithm~1 of \citet{zsd}:
\[
  P_t(s) \;\propto\; P_{t-1}(s)\,
  \exp\!\Big[-\tfrac{1}{2}(R_t - M_t(s))^{\top} \Sigma_t^{+} (R_t - M_t(s))\Big],
\]
followed by normalization, where $\Sigma_t^{+}$ replaces the inverse of \citet{zsd}
because our covariance is singular at every step (Lemma~\ref{lem:singular}); the
pseudo-determinant is the same for every $s$ and cancels. The updated posterior is used
for the next token's calibration. Each privately served token is charged the fixed budget
$b$. The curator continues while
\[
  \mathit{spent} + b \le B,
\]
and otherwise switches to a public fallback rule.

\begin{algorithm}[t]
\caption{PAC-private next-token release}
\label{alg:pac-token}
\begin{algorithmic}[1]
\Require Context $q_t$, posterior $P_{t-1}$, secret world $S$, running total $\mathrm{spent}$, budget parameters $b,B$,
public support size $k$, decoder $\in\{\text{greedy},
\text{coupled-Gumbel}(\tau)\}$
\If{$\mathrm{spent}+b>B$}
    \State \Return public fallback token
\EndIf
\State $V_t\gets\mathrm{TopK}(p_{\mathrm{pub}}(\cdot\mid q_t),k)$
\If{decoder is coupled Gumbel}
    \State Draw $u_v\overset{\mathrm{iid}}{\sim}\mathrm{Gumbel}(0,1)$ for all $v\in V_t$
    \State $v_i\gets\arg\max_{v\in V_t}
    [\log p_{S_i}(v\mid q_t)/\tau+u_v]$ for each world $i$
\Else
    \State $v_i\gets\arg\max_{v\in V_t}p_{S_i}(v\mid q_t)$ for each world $i$
\EndIf
\State $M_t : s \mapsto e_{v(s)}$ \Comment{one-hot vote map; deterministic given $U_t$}
\State $A_t\gets\mathrm{unique}(v_{1:m})$
\State $\pi_t(v)\gets\sum_{i:v_i=v}P_{t-1}(S_i)$ for $v\in A_t$
\State $C_t\gets\operatorname{diag}(\pi_t)-\pi_t\pi_t^\top$
\State $\Sigma_t \leftarrow \Sigma(P_{t-1}, M_t, b)$ on the active subspace of $C_t$ \Comment{conditional on $U_t$ under coupled decoding}
\State $R_t\gets e_{v_S}+Z_t$, where $Z_t\sim\mathcal{N}(0,\Sigma_t)$
\State $Y_t\gets\arg\max_{v\in A_t}(R_t)_v$
\State $P_t(s)\propto P_{t-1}(s)
\exp\!\left[
-\frac{1}{2}(R_t-M_t(s))^\top
\Sigma_t^{+}
(R_t-M_t(s))
\right]$, then normalize
\State $\mathrm{spent}\gets\mathrm{spent}+b$
\State \Return $Y_t$
\end{algorithmic}
\end{algorithm}
\subsection{The Token Interface}
The PAC accounting is applied to the curator's internal history $H_T$, whereas
the user observes only the released token sequence $Y_{1:T}$. The autoregressive query $q_{t+1}$ is itself a function of
the previously released tokens. Section~\ref{sec:theory} formalizes why this
token-adaptive interaction is still covered by posterior-aware composition. In particular,
\[
  I(S; Y_{1:T}) \;\le\; I(S; H_T) \;\le\; \sum_{t=1}^{T} b_t ,
\]
and for a constant per-token budget $b$,
\[
  I(S; Y_{1:T}) \;\le\; I(S; H_T) \;\le\; bT .
\]
The right-hand inequality is the composition theorem of \citet{zsd}, which bounds
$I(S; H_T)$ directly; the left is data processing, since every released token is a
deterministic function of the internal history (Lemma~\ref{lem:token}). Thus the token
stream can leak no more than the internal process. The empirical analysis later estimates
$I(S; H_T)$ from posterior entropy; it does not directly estimate the gap between the
internal history and the token channel.

\subsection{Coupled Gumbel Decoding}
\label{sec:gumbel}

Greedy voting is the direct analogue of hard-label output stabilization, but long-form
generation benefits from sampled decoding. Sampling each world independently would add
decoder randomness to the vote disagreement even when the underlying world
distributions are similar. We therefore use shared-randomness coupling.

After $q_t$ is fixed, the curator draws one set of shared Gumbel variables
\[
U_t=(u_v)_{v\in V_t},
\qquad
u_v\overset{\mathrm{iid}}{\sim}\mathrm{Gumbel}(0,1),
\]
and every world votes
\[
v_i=
\arg\max_{v\in V_t}
\left[
\frac{\log p_{S_i}(v\mid q_t)}{\tau}+u_v
\right],
\qquad \tau=1.
\]
Marginally, each world obtains a sample from its own temperature-$\tau$,
support-restricted conditional distribution, while the shared randomness preserves
coupling across similar worlds \citep{maddison2014}. Conditional on the fresh public coins, the vote map is
deterministic; because the coins are independent of $S$ and drawn only after the query
is fixed, they do not add privacy leakage. The formal accounting statement is given in
Section~\ref{sec:theory}; the generation-quality effects are evaluated separately in
the results.

Inference determinism for all non-randomized parts of the mechanism, together with
numerical and reproducibility settings, is documented in
Section~\ref{sec:setup}.

%% file: sections/04_theory.tex
\section{Privacy Analysis}
\label{sec:theory}

The mechanism of Section~\ref{sec:mechanism} instantiates posterior-aware composition
\citep{zsd} for autoregressive generation. Four features of this instantiation require additional justification in the autoregressive setting: the vote covariance is singular at every step, the
release is a token rather than the vector the accounting is defined on, the query
sequence contains the mechanism's own outputs, and coupled decoding makes the per-world
vote random. This section states the end-to-end guarantee and the results it rests on,
and closes with the identity that turns the leakage bound into a measurable quantity.
Proofs are sketched here; full arguments are in Appendix~\ref{app:proofs}.

\subsection{Setting and Main Guarantee}
\label{sec:main-guarantee}

Fix the universe $U$, the $m$-world construction and prior $P_S$ of
Section~\ref{sec:threat-model}, a per-token budget $b$, and a horizon $T$. Let
$Y_1,\dots,Y_T$ be the tokens released by Algorithm~\ref{alg:pac-token} under any
adversarial querying strategy in which the query at step $t+1$ may depend on the
released prefix $Y_{\le t}$, on public randomness, and on the adversary's own
randomness, but has no side channel to $S$.

\paragraph{Step ordering and notation.}
Each step proceeds as
\[
  q_t \ \text{fixed} \;\longrightarrow\; U_t \ \text{realized} \;\longrightarrow\;
  M_t^{U_t} \ \text{fixed} \;\longrightarrow\; C_t,\ \Sigma_t
  \;\longrightarrow\; R_t \;\longrightarrow\; Y_t .
\]
The query $q_t$ is determined by the history; the public coins $U_t$ are then realized;
and only then is the step's mechanism fixed. We write
$M_t^{U_t} : s \mapsto e_{v(s;\,q_t, U_t)}$ for the induced one-hot vote map, which is
deterministic in $s$ once $(q_t, U_t)$ are. Its covariance under the current belief
$P_{t-1}$ is
\[
  C_t \;=\; \operatorname{Cov}_{s \sim P_{t-1}}\!\big(M_t^{U_t}(s)\big)
  \;=\; \operatorname{diag}(\pi_t) - \pi_t\pi_t^\top ,
\]
and $\Sigma_t$ is the Gaussian noise covariance calibrated from $C_t$ to the budget $b$,
supported on the active subspace of $C_t$ (Lemma~\ref{lem:singular}). Under greedy
decoding $U_t$ is null and $M_t^{U_t} = M_t$; every statement below covers that case as
a degenerate instance.

\paragraph{Interaction history.}
Because the coins enter the mechanism, they enter the history:
\[
  H_t \;=\; \big\{ (M_i^{U_i},\, U_i,\, R_i) \big\}_{i=1}^{t}, \qquad H_0 = \emptyset .
\]
This is the object the belief $P_t$ conditions on, the object the composition bound
applies to, and the object whose entropy is measured in Section~\ref{sec:mi-identity}.
The voted set $A_t$ is $H_t$-measurable (Lemma~\ref{lem:support}), and each released
token $Y_t = \arg\max_{v \in A_t}(R_t)_v$ is a deterministic function of $H_t$.

\paragraph{Assumptions on the coins.}
The analysis uses exactly three properties of $U_t$:
\begin{enumerate}
  \item[(C1)] $I(S; U_t \mid H_{t-1}) = 0$;
  \item[(C2)] $U_t$ is determined, and known to the curator, before $\Sigma_t$ is
  computed;
  \item[(C3)] $U_t$ is a component of $H_t$.
\end{enumerate}
Drawing $U_t$ afresh at each step, as Algorithm~\ref{alg:pac-token} does, is a
sufficient way to obtain (C1)--(C3); it is not necessary, and in particular we make no
soundness claim against coins derived deterministically from the context
(Remark~\ref{rem:coin-order}).

\begin{theorem}[End-to-end guarantee]
\label{thm:main}
Under the conditions above, Algorithm~\ref{alg:pac-token} satisfies the following guarantee for either decoder, provided that the public coins satisfy (C1)–(C3)
\[
  I(S; Y_{1:T}) \;\le\; I(S; H_T) \;\le\; bT .
\]
Consequently, for any attack criterion whose informed-adversary prior success rate is
$\bar\delta_0$, the posterior success rate is bounded by the PAC privacy conversion at
total budget $B_T = bT$. In particular membership inference against any record has prior
$\bar\delta_0 = 1/2$, and identification of the realized world has prior
$\bar\delta_0 = 1/m$.
\end{theorem}

\begin{proof}[Proof sketch]
The chain rule decomposes $I(S; H_T)$ into per-step terms
$I(S; M_t^{U_t}, U_t \mid H_{t-1}) + I(S; R_t \mid H_{t-1}, M_t^{U_t}, U_t)$. The first
vanishes by (C1) together with Proposition~\ref{prop:autoregressive}; the second is at
most $b$ by Proposition~\ref{prop:coupled}(i), whose conditional-determinism hypothesis
holds by (C2) and whose calibration is valid despite the singular covariance by
Lemma~\ref{lem:singular}. Summing gives $I(S; H_T) \le bT$, and
Lemma~\ref{lem:token} supplies the first inequality. The attack bounds follow from the
PAC privacy conversion, whose only inputs are the prior and the total budget. Full
derivation in Appendix~\ref{app:thm-main}.
\end{proof}

The guarantee has the same form as in the classification setting: it depends only on
$(b, T, \bar\delta_0)$. What is new is that it holds for a token stream in which the
mechanism's outputs become its subsequent inputs, and that $T$ counts released tokens
rather than user-level queries.

\subsection{Token-Only Release}
\label{sec:token-release-theory}

The accounting is defined on the internal history $H_T$, which is never released.

\begin{lemma}[Token-only release]
\label{lem:token}
Suppose the curator runs Algorithm~\ref{alg:pac-token} with per-step budgets $b$, and
the adversary chooses each query as a function of the token history $(Y_{<t}, M_{<t})$
and its own randomness. Then
(i) the token-adaptive adversary is a special case of the $H$-adaptive adversary, since
$\sigma(Y_{<t}, M_{<t}) \subseteq \sigma(H_{t-1})$; and
(ii) $S \to H_T \to Y_{1:T}$ is a Markov chain, so
\[
  I(S; Y_{1:T}) \;\le\; I(S; H_T).
\]
\end{lemma}

\begin{proof}[Proof sketch]
Each $Y_t$ is $H_t$-measurable, since $A_t$ is (Lemma~\ref{lem:support}) and
$Y_t = \arg\max_{v\in A_t}(R_t)_v$. Hence the token history generates a
sub-$\sigma$-algebra of $\sigma(H_{t-1})$, giving (i), and $Y_{1:T}$ is a deterministic
function of $H_T$, giving the Markov chain and (ii) by data processing.
\end{proof}

\begin{remark}[Why the bound runs through $H_T$ rather than $R_{1:T}$]
\label{rem:not-R-alone}
Under coupled decoding it is \emph{not} valid to write
$I(S; Y_{1:T}) \le I(S; R_{1:T})$: the voted set $A_t$ is a function of $U_t$, not of
$R_t$, so $S \to R_{1:T} \to Y_{1:T}$ is not a Markov chain, and conditioning on the
coins repairs the deterministic map but not the inequality, since a conditional mutual
information need not be bounded by its unconditional counterpart. The correct object is
the joint history; see Appendix~\ref{app:not-R} for the details. Under greedy decoding
the distinction collapses, but we state everything through $H_T$ so that one argument
covers both decoders.
\end{remark}

\begin{remark}[What is and is not claimed]
\label{rem:conservatism}
The budget is charged on $H_T$ while only $Y_{1:T}$ leaves the curator, so the
accounting is conservative by whatever the $\arg\max$ discards. We do not quantify that
gap; doing so would require tracking the adversary's token-level posterior with
certified likelihoods, a different mechanism with a different calibration. The empirical
analysis of Section~\ref{sec:privacy-results} estimates $I(S; H_T)$, which by
Lemma~\ref{lem:token} upper-bounds the leakage reaching the user. The calibration
matrices $\Sigma_t$ are themselves never published, which is immaterial for the same
reason.
\end{remark}

\subsection{Singular Covariance Calibration}
\label{sec:singular-covariance}

The PAC noise determination result calibrates from an output covariance whose spectrum
may be full. Ours never is: $C_t$ has rank at most $a_t - 1$, and $C_t = 0$ under
unanimity. Calibration must therefore be defined on the active subspace, and the belief
update must use a pseudoinverse. This low-rank structure is a key advantage of our mechanism. Since
$a_t \leq \min(m,k)$, $\operatorname{rank}(C_t) \leq \min(m-1,k-1)$,
which is at most $127$ in our $m=128$, $k=200$ setting. PAC calibration
adds noise only in this active disagreement subspace, unlike PMixED,
which does not exploit this covariance structure.

\begin{lemma}[Active-subspace calibration]
\label{lem:singular}
Let $M_t^{U_t}$ take values in $\{e_v : v \in A_t\}$ with covariance $C_t$ under
$P_{t-1}$, and let $\mathcal{A} = \operatorname{range}(C_t)$. Let $\Sigma_t$ be the PAC
calibration computed on $\mathcal{A}$ and zero on $\mathcal{A}^\perp$. Then
(i) $I_{S \sim P_{t-1}}(S;\, M_t^{U_t}(S) + Z_t) \le b$ for
$Z_t \sim \mathcal{N}(0,\Sigma_t)$;
(ii) the update $P_t(s) \propto P_{t-1}(s)\exp[-\tfrac12 (R_t - y_s)^\top
\Sigma_t^{+}(R_t - y_s)]$ is the exact Bayes update given $R_t$; and
(iii) if $a_t = 1$ then $\Sigma_t = 0$ and the step leaks nothing.
\end{lemma}

\begin{proof}[Proof sketch]
All outputs $e_v$ have coordinate sum one, so their differences lie in
$\mathbf{1}^\perp$, while $\ker C_t \supseteq \operatorname{span}(\mathbf{1})$. The
projection of $M_t^{U_t}(S)$ onto $\mathcal{A}^\perp$ is therefore constant in $S$:
releasing it exactly contributes no mutual information, and (i) reduces to the noise
determination result within $\mathcal{A}$. For (ii), the pseudo-determinant factor is
identical for every $s$ and cancels in the normalization. (iii) is immediate from
$C_t = 0$. Appendix~\ref{app:singular}.
\end{proof}

\begin{remark}[Numerical floor]
\label{rem:floor}
Implementations must floor small eigenvalues upward and never truncate them to zero.
Truncating a small but nonzero eigenvalue would place zero noise along a direction in
which the worlds genuinely differ, violating (i); flooring upward only adds noise beyond
what the budget requires, and is therefore conservative. This is a soundness condition
on the implementation, not a numerical convenience.
\end{remark}

\begin{lemma}[Measurable restriction]
\label{lem:support}
$V_t$ is a deterministic function of $q_t$ and the public model, and $A_t$ is a
deterministic function of $(q_t, U_t)$ and the $m$ trained adapters. Both are therefore
$H_t$-measurable and independent of $S$ given that information. Releasing
$\arg\max_{v \in A_t}(R_t)_v$ rather than $\arg\max_{v \in V_t}(R_t)_v$ discloses
nothing beyond $H_t$.
\end{lemma}

\begin{proof}
$V_t$ is computed from the public model on a public context. $A_t$ is the set of votes
of all $m$ worlds, not of the realized one, so it is determined by $(q_t, U_t)$ and the
adapters; the informed adversary knows every adapter and, by (C3), observes every coin,
so it can compute $A_t$ itself. Conditioning on a quantity the adversary can already
compute adds no information.
\end{proof}

The scope of this lemma should be read narrowly. It says the restriction to the public
support opens no inference channel; it does not say that membership in the support is
private. Off-support tokens are structurally non-emittable, a guarantee of a different
kind that is not part of the mutual-information accounting
(Section~\ref{sec:canary-results}).

\subsection{Sequential Composition}
\label{sec:composition}

Posterior-aware composition assumes $M_t \perp S \mid H_{t-1}$. In our setting the query
is not merely adaptive but self-generated: $q_{t+1}$ contains the mechanism's own
outputs.

\begin{proposition}[Autoregressive queries are admissible]
\label{prop:autoregressive}
Let $q_{t+1} = g(q_t, Y_t)$ for the deterministic append map $g$, with the prompt chosen
by the adversary from $(Y_{<t}, M_{<t})$ and its own randomness $\omega \perp S$. Then
$M_{t+1} \perp S \mid H_t$, so per-step budgets compose linearly and
$I(S; H_T) \le \sum_{t=1}^{T} b_t$.
\end{proposition}

\begin{proof}
$Y_t$ is a deterministic function of $H_t$, and $q_t \in H_t$ by induction, so $q_{t+1}$
is $H_t$-measurable. With $\omega \perp S$, the query $M_{t+1}$ depends on $S$ only
through $H_t$, which is the required condition.
\end{proof}

\begin{remark}[Feedback does not amplify]
Feeding the mechanism's output back into its input creates no unaccounted loop: the
released token was charged at the step that produced it, and conditioning the next query
on it is exactly the adaptivity the composition theorem tolerates. The consequence is
not a weaker bound but a faster clock, since $T$ counts tokens. Likewise the stop rule
$\mathit{spent} + b \le B$ is $H_{t-1}$-measurable, so halting discloses nothing beyond
the public budget; realized generation length may differ from the nominal length and is
reported as such.
\end{remark}

\subsection{Coupled Decoding}
\label{sec:coupled-decoding}

The noise determination result privatizes a \emph{deterministic} map $s \mapsto M_t(s)$.
\citet{zsd} meet this precondition by derandomizing training and inference outright.
Coupled Gumbel decoding reintroduces randomness into the vote, so the precondition must
be re-established rather than inherited.

\begin{proposition}[Coupled decoding preserves the accounting]
\label{prop:coupled}
Assume (C1)--(C3), and let the votes be
$v_i = \arg\max_{v \in V_t}[\log p_{S_i}(v \mid q_t)/\tau + u_v]$. Then
(i) $I(S; R_t \mid H_{t-1}, M_t^{U_t}, U_t) \le b$;
(ii) $I(S; M_t^{U_t}, U_t \mid H_{t-1}) = 0$; and
(iii) the belief update conditioned on $(H_{t-1}, M_t^{U_t}, U_t)$ is the exact posterior
$P_{S \mid H_t}$, so posterior-aware composition applies unchanged.
\end{proposition}

\begin{proof}[Proof sketch]
For (i), (C2) makes $s \mapsto M_t^{U_t}(s)$ deterministic once the coins are realized,
which is the precondition of the noise determination result, applied here to the
conditional mechanism whose covariance is $C_t$. For (ii), (C1) gives
$I(S; U_t \mid H_{t-1}) = 0$, and given $(H_{t-1}, U_t)$ the map $M_t^{U_t}$ is a
function of $q_t$ and the public adapters, with $q_t$ depending on $S$ only through
$H_{t-1}$ (Proposition~\ref{prop:autoregressive}). For (iii), the likelihood used is the
exact conditional likelihood given $(H_{t-1}, M_t^{U_t}, U_t)$, all components of $H_t$
by (C3). No independence between $U_t$ and $q_t$ is used anywhere.
Appendix~\ref{app:coupled}.
\end{proof}

\begin{remark}[What the coupling requires of the coins]
\label{rem:coin-order}
Conditions (C1)--(C3) are weaker than the implementation we adopt. Coins derived
deterministically from the context, $U_t = f(q_t, H_{t-1})$ for public measurable $f$,
satisfy all three: such $U_t$ is $\sigma(H_{t-1}, M_t)$-measurable, so
$I(S; M_t^{U_t}, U_t \mid H_{t-1}) = I(S; M_t \mid H_{t-1}) = 0$, giving (C1); it is
pinned once $(H_{t-1}, q_t)$ are, giving (C2); and public $f$ makes it recoverable,
giving (C3). Nor is an adversary who predicts the coins and searches for a favourable
prompt a problem, since the bound in (i) is conditional on $(M_t^{U_t}, U_t)$ and holds
for every realization. What (C1) does exclude are coins seeded from the world index,
from the realized world's logits, or from any quantity depending on $S$ other than
through $H_{t-1}$. Appendix~\ref{app:coins} gives the verification and one caveat
concerning entropy-charging variants of the accounting.
\end{remark}

\begin{remark}[Marginal fidelity]
By the Gumbel-max construction each world's vote is marginally an exact sample from its
own temperature-$\tau$, support-restricted conditional. The coupling alters the joint
law of the votes, which preserves cross-world agreement, without perturbing any world's
marginal; its effect on the estimator of Section~\ref{sec:mi-identity} is additional
variance across sessions, not bias.
\end{remark}

\subsection{Posterior Entropy and Mutual Information}
\label{sec:mi-identity}

The results above bound leakage. Because the secret space is finite and the tracked
belief is the true posterior, we can also estimate it.

\begin{proposition}[Posterior-entropy identity]
\label{prop:mi-identity}
With $S \sim \mathrm{Uniform}(\mathcal{S})$, $|\mathcal{S}| = m$, $P_T$ the belief
tracked by Algorithm~\ref{alg:pac-token}, and $H_T$ as above,
\[
  I(S; H_T) \;=\; \log m \;-\; \mathbb{E}\big[H(P_T)\big],
\]
the expectation being over the distribution of $H_T$, which includes the coins
$U_{1:T}$.
\end{proposition}

\begin{proof}
$I(S; H_T) = H(S) - H(S \mid H_T)$ and $H(S) = \log m$ by uniformity. By
Proposition~\ref{prop:coupled}(iii), $P_T = P_{S \mid H_T}$, so
$H(S \mid H_T) = \mathbb{E}[H(P_T)]$, the expectation being over the same history the
conditioning is on.
\end{proof}

Three points must be kept separate when this identity is used empirically, and the
measurements of Section~\ref{sec:privacy-results} observe them.
\begin{enumerate}
  \item \textbf{The identity is exact; the expectation is estimated.}
  $\mathbb{E}[H(P_T)]$ is approximated by an average over finitely many sessions, so
  what we report is a direct estimate with quantified sampling error, not an exact
  measurement of the mutual information.
  \item \textbf{The estimand is the same object throughout.} Both sides concern $H_T$,
  which contains the coins; this is the quantity bounded by Theorem~\ref{thm:main} and
  shown by Lemma~\ref{lem:token} to dominate token-channel leakage, so the estimate, the
  analytic bound, and the user-facing guarantee refer to one object rather than three.
  \item \textbf{The coins contribute variance, not leakage.} Different coin draws give
  different vote patterns and posterior trajectories; by
  Proposition~\ref{prop:coupled}(ii) they carry no information about $S$, so they widen
  the sampling distribution of the estimate without shifting its target.
\end{enumerate}

This is a stronger form of empirical verification than an attack: an attack lower-bounds
leakage by exhibiting an adversary, whereas the identity gives the leakage itself up to
Monte Carlo error. It is available here because the $m$-world construction keeps the
secret space enumerable.

%% file: sections/05_experimental_setup.tex
\section{Experimental Setup}
\label{sec:setup}

\subsection{Models, Data, and World Construction}
\label{sec:models-data}

The private universe consists of WikiText-103 articles \citep{merity2017} together with
canary records used only in Section~\ref{sec:canary-results}; a disjoint set of articles is held
out as non-members. Worlds are formed by the $m/2$ assignment of
Section~\ref{sec:threat-model}, and each world is a LoRA adapter \citep{hu2022} over a
frozen public GPT-2-small base \citep{radford2019}. Training is identical across worlds
--- same optimizer, learning rate, and fixed epoch count, with no early stopping --- so
that no world receives a configuration chosen on the basis of its own data.
Table~\ref{tab:setup} gives the full configuration.

All experiments involving the PAC world construction use a single collection,
Collection~A (construction seed $101$). The PMixED comparison of
Section~\ref{sec:pmixed-results} necessarily uses a different partition of the same universe,
described in Section~\ref{sec:baselines}. The implications of using one collection are
discussed in Section~\ref{sec:discussion}.

\begin{table}[t]
\centering
\caption{System and mechanism configuration. All mutual-information budgets are in nats.}
\label{tab:setup}
\begin{tabular}{ll}
\toprule
Component & Value \\
\midrule
Worlds & $m = 128$; each record in exactly $m/2 = 64$ \\
Prior & uniform over worlds; per-record membership prior exactly $1/2$ \\
Universe $U$ & $2{,}048$ WikiText-103 articles $+$ $16$ canaries $= 2{,}064$ \\
Held-out non-members & $512$ articles \\
Base model & GPT-2-small (124M), frozen, \texttt{fp32} \\
Adapters & LoRA $r = 8$, $\alpha = 16$, dropout $0$, $\{c_\mathrm{attn}, c_\mathrm{proj}\}$, $\approx\!0.6$M params/world \\
Training & AdamW, lr $2\times10^{-4}$, $2$ epochs, identical for every world \\
Public support & $k = 200$ (argmax and mass coverage at $k \in \{50,200,1000\}$ in Section~\ref{sec:stability-utility}) \\
Decoders & greedy; coupled Gumbel at $\tau = 1.0$ \\
Budget grid & $b \in \{2^{-4}, 2^{-8}, \dots, 2^{-32}\}$, $8$ values \\
Numerics & \texttt{float64} from vote covariance through posterior update \\
Determinism & \texttt{use\_deterministic\_algorithms(True)}, TF32 off, fixed CuBLAS workspace \\
World collection & A, construction seed $101$ \\
\bottomrule
\end{tabular}
\end{table}

\paragraph{Temperature.}
The coupling temperature interpolates between the two decoders: as $\tau \to 0$ the
Gumbel perturbation vanishes and each world's vote returns to its in-support argmax,
while larger $\tau$ makes each world's vote a more stochastic sample from its own
support-restricted conditional. All worlds share the same coins at every step, so $\tau$
controls per-world stochasticity without independently perturbing the worlds relative to
one another. We use $\tau = 1.0$ throughout; Section~\ref{sec:generation-results} reports
$\tau = 0.8$ and two alternative couplings.

\paragraph{Randomness and conditional determinism.}
What the calibration result requires is that the vote map $s \mapsto M_t(s)$ be
deterministic given the information the noise is calibrated against
(Section~\ref{sec:theory}); it does not require that inference be free of randomness.
Under greedy decoding this means the forward pass and in-support argmax must be
reproducible, so that the map is a fixed function of the context. Under coupled Gumbel
decoding the map is deterministic \emph{conditional on} the public coins $U_t$, which are
themselves legitimate randomness: they are independent of the secret, recorded in the
history, and conditioned on by the calibration and the belief update. The same is true of
the calibration noise $Z_t$, which is the mechanism's output randomness rather than part
of the map. Accordingly we fix the sources that would otherwise make the map irreproducible
--- \texttt{use\_deterministic\_algorithms(True)}, TF32 disabled, a fixed CuBLAS workspace
--- and verify that per-world votes are bitwise reproducible under fixed coins before any
privacy claim is made. Training-time nondeterminism requires no such control, since the
informed adversary is assumed to know all $m$ trained adapters.

\subsection{Evaluation Protocol}
\label{sec:eval-protocol}

Three regimes answer different questions, and are summarized in
Table~\ref{tab:eval-regimes}. \emph{Teacher-forced} evaluation supplies the true prefix as
context, so it isolates prediction quality from autoregressive drift and degeneracy;
this is where the primary utility results are measured. \emph{Per-sequence generation}
runs the mechanism autoregressively over a grid of budgets and lengths under greedy
decoding, and is where sequence-level attacks are evaluated. \emph{Persistent-secret
sessions} draw one world and serve a long stream against it, which is the setting the
composition analysis actually concerns and the only regime in which the direct
mutual-information estimate is available. The stability census reports argmax-vote
statistics and is therefore greedy; the coupled-Gumbel comparison of vote agreement
comes from the paired $T = 64$ ablation of Section~\ref{sec:generation-results}, in which prompts,
worlds, and noise seeds are held fixed and only vote formation differs. The coupling
ceiling $\beta$ is computed from the world distributions themselves and does not depend
on the decoder. Support coverage is reported in two forms, since the relevant notion differs by decoder: whether a world's unrestricted argmax survives truncation, and how much of its probability mass does.
\begin{table}[t]
\centering
\caption{Evaluation regimes. Trials draw a fresh realized world and fresh noise.}
\label{tab:eval-regimes}
\small
\setlength{\tabcolsep}{5pt}
\renewcommand{\arraystretch}{1.1}
\begin{tabularx}{\linewidth}{>{\raggedright\arraybackslash}p{2.8cm}
                                  >{\raggedright\arraybackslash}p{1.8cm}
                                  >{\raggedright\arraybackslash}X
                                  >{\raggedright\arraybackslash}p{3.2cm}}
\toprule
Regime & Decoder & Scale & Primary use \\
\midrule
Stability census & greedy & 65,053 held-out positions & unanimity, $a_t$, argmax and mass coverage, $\beta$ \\

Teacher-forced accuracy & both & 8,176 positions (16,352 with replication) & utility, noise sensitivity \\

Flatness diagnostic & both & 2,048 positions $\times$ 8 coin draws & budget--flatness relation \\

Per-sequence sweep & greedy & 8 budgets $\times$ $T \in \{16,64,256\}$; 1,000 trials (100 at $T=256$) & MIA, world identification \\

Sessions & Gumbel & $b \in \{2^{-16},2^{-20},2^{-24}\}$; 100 sessions $\times$ 4,096 tokens & direct MI, utility over long runs \\

Canary emulation & greedy and Gumbel & 2 canaries $\times$ 3 budgets + no-noise reference ; 2,000 trials per cell, each decoder & membership inference on forced memorization \\

PMixED comparison & Gumbel & 2,048 WT-103 test positions & matched-guardrail comparison \\
\bottomrule
\end{tabularx}
\end{table}

\paragraph{Realized length.}
The public stop rule halts serving once the budget is exhausted, so realized sequence
length can fall short of the nominal one; at $b = 2^{-4}$ the cap binds at $16$ tokens.
We report realized rather than nominal lengths, and mark cells in which the cap binds. A
cell whose nominal length is $256$ but whose realized length is $16$ looks artificially
healthy on every degeneracy metric, which is precisely the trap this convention avoids.

\paragraph{Uncertainty.}
Generation statistics use a sequence-level (cluster) bootstrap over $1{,}000$ resamples,
since tokens within a sequence are not independent. For sessions we report the standard
error across sessions. Cells at different horizons within one budget share sessions, so
agreement across horizons is a consistency check rather than a confidence statement; the
three budgets use independent session seeds and do give independent estimates.

\paragraph{Session protocol.}
A session consists of 64 queries against the same persistent secret.
Each query uses a fresh 32-token prompt and releases 64 tokens.
Thus, $T=4096$ denotes the cumulative number of releases across the
session rather than the length of a single GPT-2 context.

\subsection{Utility Metrics}
\label{sec:utility-metrics}

Held-out next-token accuracy under teacher forcing is the primary utility metric and the
one against which every privacy--utility statement in this paper is made. For generated
text we report rep-3 and distinct-1 to capture degeneracy and lexical diversity. Both
vary strongly with sequence length for \emph{any} text, human text included, so all
comparisons use references generated at the same length, and every baseline is evaluated
under the same decoder as the mechanism it is compared against --- scoring a sampled
mechanism against an argmax baseline charges the difference between decoders as a
privacy cost, which in our early runs amounted to roughly $12$ accuracy points of pure
artifact.

\paragraph{Why perplexity is not our primary utility metric.}
Two considerations argue against it. First, perplexity of self-generated text rewards
predictability rather than quality: a degenerate loop attains a perplexity near $1.7$,
and a more diverse model can score worse than a less diverse one. Second, and specific
to a token-release interface, the release distribution of our mechanism puts zero mass
outside the voted set $A_t$; whenever the ground-truth token falls outside $A_t$ the
target-token likelihood is exactly zero, and the corresponding NLL term is determined by
whatever numerical floor the implementation applies rather than by model quality. A
perplexity aggregated over such positions is partly a statement about the floor, and how
much so depends on the floor and on how often the case arises.

Teacher-forced accuracy therefore remains our main utility measure throughout. This does
not make perplexity useless, provided it is clear what distribution the likelihood is
taken under. Where we report it descriptively for generated text, alongside rep-3 and
distinct-1, the text is scored under a fixed \emph{reference language model} --- not
under the mechanism's token-release distribution --- so it measures how typical the
generated string is rather than how likely the mechanism was to produce it, and the
zero-mass problem above does not arise. In Section~\ref{sec:pmixed-results} we additionally
report perplexity for PMixED in the usual sense, since its release is a full
distribution over the vocabulary and target-token NLL is well defined under it.

\subsection{Privacy Metrics and Attacks}
\label{sec:privacy-metrics}

For each $(b, T)$ we report the bound implied by Theorem~\ref{thm:main}, a function of
the total budget and the prior only, at the two relevant priors: $1/2$ for membership of
an individual record and $1/m$ for identification of the realized world. Empirically,
record-level membership inference uses the Bayes-optimal rule under the tracked
posterior --- a record is declared a member when the posterior mass of worlds containing
it exceeds one half --- and world identification takes the posterior mode. Both attacks
read the curator's own posterior, which by Proposition~\ref{prop:coupled}(iii) is the
true posterior given the internal history; they are an upper envelope on any adversary
restricted to the token stream rather than a realistic attacker model, and we report
them as such.

\paragraph{Direct mutual-information estimation.}
Sessions estimate $I(S; H_T)$ through Proposition~\ref{prop:mi-identity},
$I(S; H_T) = \log m - \mathbb{E}[H(P_T)]$, by averaging posterior entropy over sessions.
The identity is exact; the expectation is estimated from finitely many sessions, so we
report a direct estimate with its sampling error and never describe it as an exact
measurement of mutual information. By Lemma~\ref{lem:token} the estimate upper-bounds
the leakage reaching the user through the token channel.

\paragraph{Canary protocol.}
The $16$ canaries are of two kinds --- out-of-distribution sequences (C-OOD) and
sequences built from common tokens (C-COMMON) --- injected at repetition counts from $1$
to $64$. Organic memorization is measured by comparing member and non-member worlds at
the canary trigger in log-probability, rank, and vote share. A separate capacity pilot
trains single canary-bearing adapters against clean controls on identical data, at
$r = 8$ with $2$ epochs and at $r = 32$ with $4$ epochs, with repetitions up to $4{,}096$,
to establish how much repetition would be needed for a canary to reach the argmax. The
controlled experiment of Section~\ref{sec:canary-results} instead \emph{forces}
perfect memorization into exactly the member worlds at inference time, leaving weights
untouched, and measures the resulting membership advantage over $2{,}000$ trials for each of two C-COMMON canaries at $b \in \{2^{-4}, 2^{-8}, 2^{-16}\}$ under both decoders, against a no-noise reference in which the memorization is not privatized at all. Each trial releases a single token at the canary trigger and records whether it is the canary payload; we report the advantage with its standard error ($\approx 0.022$ at $2{,}000$ trials).

\subsection{Baselines}
\label{sec:baselines}

Three reference points bracket the privacy--utility question. The \emph{public} baseline
is the frozen base model with no private adaptation. The \emph{full-$U$} baseline is a
single adapter trained on all $2{,}048$ articles, giving the non-private ceiling. The
\emph{$b = \infty$} baseline runs the full mechanism with no calibration noise. The two
gaps are then separately meaningful: full-$U$ to $b = \infty$ is subsampling error, and
$b = \infty$ to a finite budget is noise cost. Generation experiments additionally use
human continuations and decoder-matched public and non-private continuations at each
length.

\paragraph{PMixED.}
Section~\ref{sec:pmixed-results} compares against PMixED \citep{flemings2024}, the existing
private-prediction method for language models, implemented from the authors' own ensemble construction:
their RD-mollification mechanism, Poisson-amplified RDP accounting,
$\beta$ selection from the per-query privacy constraint, and
$\lambda$ selection by bisection via Eq.~5.. Its privacy argument requires each record to
influence exactly one expert, so the $m/2$ worlds cannot be reused; we partition the same
$2{,}064$-record universe into $N \in \{8, 16, 32\}$ disjoint shards and train with the
authors' schedule. The implementation is validated against the authors' default
configuration before any comparison is drawn (Section~\ref{sec:pmixed-results}).
We evaluate both methods on the official WikiText-103 test split, which is disjoint from the training data used for either comparison model. Both methods are compared at matched provable
membership-inference bounds, under both accounting conventions discussed in
Section~\ref{sec:pmixed-results}.

\subsection{Computational Cost}
\label{sec:cost}
Collapsing calibration onto the distinct votes removes the dependence on vocabulary
size. Calibrating and sampling Gaussian noise over the full vocabulary costs
$O(md\min(m,d) + d^{2})$ operations per release \citep{zsd}, which at $m = 128$ worlds
and $d = |\mathcal{V}| = 50{,}257$ is $3.35 \times 10^{9}$. Restricting to the $k = 200$
candidate set and then to the $a_t$ distinct votes reduces this to
$O(mk + a_t^{3} + m a_t)$, or $2.59 \times 10^{4}$ at the typical $a_t = 2$ measured in
Section~\ref{sec:stability-utility} --- a factor of $1.3 \times 10^{5}$ fewer operations.

Both expressions exclude the $m$ model evaluations themselves, and that term is what
dominates in practice. On a single NVIDIA H100 NVL (94\,GB) with PyTorch 2.7.0 and CUDA
12.6, the batched 128-world forward pass takes 29.71\,ms per released token under coupled
Gumbel decoding, while the entire mechanism --- vote collapse, calibration, noise
sampling and the posterior update --- adds 0.223\,ms, or 0.74\% of the total. Throughput
is 33.4 tokens/s at a peak GPU memory of 2.23\,GB.

The cost of privacy here is therefore the cost of evaluating an ensemble, not of
privatizing its output. Serving 128 worlds is $5.0\times$ slower than a single
non-private forward pass (5.92\,ms), which is $25.5\times$ cheaper than evaluating the
worlds sequentially; batching, rather than the privacy mechanism, is what makes $m = 128$
practical.

%% file: sections/06_results.tex
\section{Results}
\label{sec:results}

All results using the PAC world construction use Collection~A; PMixED uses its required
disjoint-shard ensemble described in Section~\ref{sec:baselines}. Greedy and coupled
Gumbel are not treated as interchangeable utility baselines; direct cross-decoder
comparisons are restricted to explicitly paired decoding ablations and robustness
checks.

\subsection{Ensemble Stability and Utility}
\label{sec:stability-utility}

\paragraph{How often do the worlds agree?}
Table~\ref{tab:census} reports the stability census over held-out positions. At
$k = 200$, all 128 worlds cast the same greedy vote at 83.4\% of positions, the median
number of distinct votes is 1 and the 95th percentile is 2. Support coverage is reported
in both forms. Across world-position pairs, 98.6\% of unrestricted world argmaxes fall
inside $V_t$; but the support retains only 88.5\% of a world's probability mass on
average (median 93.1\%, 5th percentile 61.5\% over the 20{,}000 per-position records
stored by the census, whose mean matches the full-set mean to $4\times10^{-4}$).

\begin{table}[t]
\centering
\caption{Stability census, greedy votes. $n = 65{,}053$ positions ($19{,}982$ at
$k = 1000$). Coverage figures are averaged over world-position pairs: argmax coverage is
the fraction of world argmaxes falling inside $V_t$, mass coverage the share of a
world's distribution retained by $V_t$. $\beta = \sum_v \min_s p_s(v)$ is the maximum
probability that all worlds agree under any coupling.}
\label{tab:census}
\begin{tabular}{lccccccc}
\toprule
$k$ & unanimity & median $a_t$ & p95 $a_t$ & argmax cov. & mass cov. & $\beta$ (mean) \\
\midrule
50   & 0.835 & 1 & 2 & 0.974 & 0.799 & 0.881 \\
200  & 0.834 & 1 & 2 & 0.986 & 0.885 & 0.878 \\
1000 & 0.827 & 1 & 2 & 0.995 & 0.953 & 0.874 \\
\bottomrule
\end{tabular}
\end{table}

The two coverage notions diverge because the support is chosen by the public model:
truncation almost never removes a world's single preferred token, but it removes an
appreciable share of the tail. For the greedy decoder only the first matters; for
sampling, the second is the relevant figure, and $k = 200$ retains 88.5\% of mass
against 95.3\% at $k = 1000$. We use $k = 200$ throughout, so this is a cost we pay
rather than one we avoid.

\paragraph{What does privacy cost in accuracy?}
Table~\ref{tab:accuracy} decomposes held-out next-token accuracy into the gain from
private data, the cost of subsampling, and the cost of noise. The mechanism retains 74.3\% of the fine-tuning gain under greedy and 84.4\% under
coupled Gumbel at $b = 2^{-32}$, a budget that bounds
membership-inference success below 51.08\% after $10^6$ released tokens.

\begin{table}[t]
\centering
\caption{Teacher-forced next-token accuracy (\%). Greedy $n = 8{,}176$ positions,
Gumbel $n = 16{,}352$. Decoders are not comparable to each other. Retained gain is
(PAC $-$ public)/(full-$U$ $-$ public).}
\label{tab:accuracy}
\begin{tabular}{lcc}
\toprule
& greedy & coupled Gumbel \\
\midrule
public GPT-2                & 38.16 & 25.56 \\
full-$U$ (non-private)      & 43.92 & 31.49 \\
$b = \infty$ (subsampled, no noise) & 42.81 & 	30.68 \\
PAC at $b = 2^{-32}$        & 42.44 & 30.57 \\
\midrule
fine-tuning gain            & $+5.76$ & $+5.93$ \\
subsampling cost            & $-1.11$ & $-0.81$ \\
noise cost                  & $-0.37$ & $-0.12$ \\
retained gain               & 74.3\% & 84.4\% \\
\bottomrule
\end{tabular}
\end{table}

Accuracy is flat across the entire budget sweep: 42.43--42.87\% (greedy) and
30.50--30.71\% (Gumbel) from $b = 2^{-4}$ to $b = 2^{-32}$,\footnote[2]{
The same Gumbel votes are used across budgets, so the comparison isolates the
effect of calibration noise. In the wider sweep, $b = 2^{-2}, 2^{0}, 2^{2},
2^{4}$ give $30.78$, $30.69$, $30.68$, and $30.68\%$, respectively.
At $b = 2^{4}$, the released token matches the $b=\infty$ release at every
position, and the two accuracies agree to five decimal places ($30.68126\%$).
} against standard errors of
0.55 and 0.36 points respectively, with no trend in $b$. The fine-tuning gain is similar
under the two decoder-matched evaluations, $+5.76$ and $+5.93$ points.

The noise cost is $-0.12$ points under Gumbel and $-0.37$ under greedy, both well inside
one standard error ($0.36$ and $0.55$ points), so neither is a measurable cost.

\paragraph{Effect of the privacy budget}
Flatness over a $2^{28}$ range invites the objection that $b$ never reaches the
mechanism. Figure~\ref{fig:flatness} and Table~\ref{tab:flatness} refute it directly. The calibrated noise scale
spans four orders of magnitude and tracks $1/\sqrt{b}$; flips relative to the noiseless
release rise from 26\% to 52\% on positions where the worlds disagree. Yet overall
accuracy moves by $-0.0012$.
\begin{figure}[t]
\centering
\includegraphics[width=\linewidth]{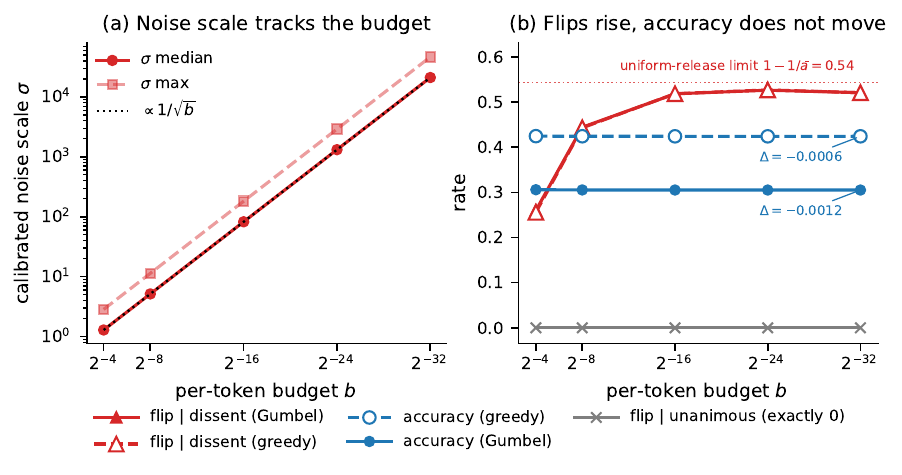}
\caption{The budget drives the mechanism, but accuracy does not follow. (a) Calibrated
noise scale against the per-token budget, with a $1/\sqrt{b}$ reference line. (b) Over the
same range, flips on disagreeing positions rise from 26\% to 52\% and saturate near the
uniform-release limit $1-1/\bar{a}$ (dotted), while flips on unanimous positions stay
exactly zero and accuracy is almost unchanged; both decoders shown.}
\label{fig:flatness}
\end{figure}
\begin{table}[t]
\centering
\caption{Budget-flatness diagnostic, coupled Gumbel, $2{,}048$ teacher-forced positions
$\times$ 8 coin draws. ``flip'' is disagreement with the $b = \infty$ release (This diagnostic uses a separate 2,048-position subset, so absolute
accuracy levels are not directly comparable to Table~\ref{tab:accuracy}).}
\label{tab:flatness}
\begin{tabular}{lcccccc}
\toprule
$b$ & $\sigma$ median & $\sigma$ max & accuracy & flip (all) & flip $\mid$ unanimous & flip $\mid$ dissent \\
\midrule
$\infty$   & 0        & 0        & 0.3063 & 0.0000 & 0.0000 & 0.0000 \\
$2^{-4}$   & 1.29     & 2.84     & 0.3058 & 0.0419 & 0.0000 & 0.2600 \\
$2^{-8}$   & 5.15     & 11.3     & 0.3053 & 0.0715 & 0.0000 & 0.4442 \\
$2^{-16}$  & 82.3     & 182      & 0.3051 & 0.0835 & 0.0000 & 0.5181 \\
$2^{-24}$  & $1.32\times10^{3}$ & $2.91\times10^{3}$ & 0.3051 & 0.0848 & 0.0000 & 0.5262 \\
$2^{-32}$  & $2.11\times10^{4}$ & $4.65\times10^{4}$ & 0.3051 & 0.0839 & 0.0000 & 0.5206 \\
\bottomrule
\end{tabular}
\end{table}

Three measured facts account for the flatness, and their arithmetic closes exactly.
Unanimous positions --- 83.9\% of the diagnostic set --- have zero output variance, draw
no noise, and flip at a rate of exactly $0.0000$ at every budget. On the remaining
positions the flip rate saturates near the uniform-release limit $1 - 1/\bar{a}$ (mean
$\bar{a} = 2.19$ gives 0.543 against an observed plateau of 0.52), so tightening $b$
past $2^{-16}$ cannot flip more. And flips are close to accuracy-neutral, because every
candidate in the voted set is some world's vote: dissent-position accuracy moves only
from 0.1473 to 0.1400. A separate diagnostic explains this flatness. When dissent positions are split by whether
the ensemble majority is correct, tightening $b$ from $\infty$ to $2^{-32}$ changes accuracy
by $-34.1$ points on the majority-correct subset and $+31.8$ points on the majority-wrong
subset, with both approaching the uniform-release limit $1/\bar{a} \approx 0.45$.
These opposite effects largely cancel. The balance is reflected by
$f=\Pr[\text{majority correct}\mid\text{dissent, true token voted}]
=0.520\pm0.018$, close to the $0.5$ cancellation point and far from the
$f\approx1$ expected for a high-accuracy classifier ensemble. In addition, $70\%$ of
dissent positions have no vote for the true token and therefore remain incorrect at every budget. The decomposition
$\Delta\text{acc} = \text{dissent fraction} \times \Delta\text{acc}\!\mid\!\text{dissent}$
gives $0.1611 \times (-0.0072) = -0.00116$, matching the observed $-0.00116$. Greedy
behaves the same way (dissent fraction 0.1460, overall $\Delta\text{acc} = -0.00060$,
dissent flips $0.256 \to 0.521$).

\subsection{Privacy Under Autoregressive Composition}
\label{sec:privacy-results}

\paragraph{Per-sequence attacks.}
Across all 24 cells of the greedy sweep we observe no statistically significant violation
of the provable bound (Table~\ref{tab:sweep}).  Record-level MIA is substantially less sensitive than world identification.
For $b\le 2^{-8}$ it remains near chance, while at $b=2^{-4}$ it rises to
approximately $0.52$. World identification is the more sensitive metric,
reaching about $5\times$ the $1/128$ prior at the loosest budget. This is structural rather than surprising: a record lies in exactly 64
of 128 worlds, so membership is a coarse function of world identity.

\begin{table}[t]
\centering
\caption{Per-sequence sweep, greedy, 1{,}000 trials (100 at $T = 256$). Representative
cells. Realized length is shorter than nominal wherever the stop rule binds, marked
$\dagger$.}
\label{tab:sweep}
\begin{tabular}{llccccc}
\toprule
$T$ & $b$ & realized len. & world-ID (bound) & MIA (bound) & unanimity \\
\midrule
64  & $2^{-4}$  & 16$^\dagger$ & 0.024 (0.894) & 0.519 (1.000) & 0.827 \\
64  & $2^{-8}$  & 64  & 0.010 (0.130) & 0.508 (0.838) & 0.888 \\
64  & $2^{-16}$ & 64  & 0.007 (0.012) & 0.501 (0.522) & 0.887 \\
64  & $2^{-32}$ & 64  & 0.011 (0.008) & 0.500 (0.500) & 0.886 \\
256 & $2^{-8}$  & 256 & 0.010 (0.337) & 0.507 (1.000) & 0.961 \\
\bottomrule
\end{tabular}
\end{table}

Two readings require care. Apparent excesses at tight budgets --- world-ID 0.011 against
a bound of 0.008 --- are Monte Carlo noise around the $1/128 = 0.0078$ prior, which the
bound collapses onto as $b$ tightens; the defensible statement is that empirical world
identification is indistinguishable from prior guessing at $b \le 2^{-12}$. And measured
unanimity rises with sequence length (0.825, 0.887, 0.961 at $T = 16, 64, 256$), which is
a consequence of greedy degeneracy rather than of the mechanism; see
Section~\ref{sec:generation-results}.

\paragraph{Direct estimation of the leakage.}
Under a persistent secret we estimate $I(S; H_T)$ from posterior entropy via
Proposition~\ref{prop:mi-identity} rather than probing it with an attack.
Table~\ref{tab:session} reports 15 cells: three budgets spanning $256\times$, five
horizons spanning $64\times$, 100 sessions each.

\begin{table}[t]
\centering
\caption{Persistent-secret sessions, coupled Gumbel, 100 sessions per budget. Measured
$I(S;H_T) = \log m - \mathbb{E}[H(P_T)]$, estimated from the session average.}
\label{tab:session}
\begin{tabular}{llccccc}
\toprule
$b$ & $T$ & charged (nats) & measured $I$ & measured/charged & MIA (bound) & world-ID (bound) \\
\midrule
$2^{-16}$ & 64   & $9.77\times10^{-4}$ & $1.83\times10^{-4}$ & 18.8\% & 0.504 (0.522) & 0.000 (0.012) \\
$2^{-16}$ & 1024 & $1.56\times10^{-2}$ & $2.67\times10^{-3}$ & 17.1\% & 0.503 (0.588) & 0.000 (0.028) \\
$2^{-16}$ & 4096 & $6.25\times10^{-2}$ & $1.08\times10^{-2}$ & 17.3\% & 0.506 (0.675) & 0.030 (0.056) \\
$2^{-20}$ & 4096 & $3.91\times10^{-3}$ & $6.72\times10^{-4}$ & 17.2\% & 0.502 (0.544) & 0.010 (0.017) \\
$2^{-24}$ & 4096 & $2.44\times10^{-4}$ & $4.23\times10^{-5}$ & 17.3\% & 0.497 (0.511) & 0.010 (0.010) \\
\bottomrule
\end{tabular}
\end{table}

The ratio of measured to charged leakage is 17.0--18.8\% across all 15 cells. At
$T = 4096$ the three budgets use independent session seeds and give three independent
estimates of the same ratio: 17.311\%, 17.209\% and 17.318\%, with a standard deviation
of 0.061 percentage points. Leakage does accumulate --- world identification reaches
0.030 at $b = 2^{-16}$, $T = 4096$, about 2.5 standard errors above the $1/128$ prior ---
while remaining under bound.

Two qualifications carry through from Section~\ref{sec:theory}. What is estimated is
$I(S; H_T)$, the leakage of the internal history, which upper-bounds the leakage of the
released token stream; we do not estimate the gap between the two. And the identity is
exact while the expectation is not, so these are direct estimates with sampling error,
not exact measurements. The spread across cells within a budget is a consistency check,
since those cells share sessions; only the three $T = 4096$ estimates are mutually
independent.

\paragraph{Where the over-charge comes from.}
Measured leakage is constant at roughly 17\% of the charged budget over a $256\times$
range in $b$ and a $64\times$ range in $T$. Fixed-$b$ accounting therefore over-charges
by a factor of $1/0.173 \approx 5.8$. The natural explanation is that a unanimous token
has zero output variance, hence zero noise and zero conditional leakage, but is charged
$b$ regardless --- which predicts a ratio equal to the dissent fraction. The dissent
fraction under coupled Gumbel is 16.8--17.4\% in the paired $T = 64$ ablation and 16.1\%
on teacher-forced positions, bracketing the measured 17.3\%. We report this as a
consistent account rather than a demonstrated mechanism: the session runs do not record
per-token unanimity, so the two quantities are measured on different protocols and are
not matched.

\subsection{Generation Quality and Decoding}
\label{sec:generation-results}

\paragraph{Greedy collapse is inherited, not caused.}
Figure~\ref{fig:degeneracy} tracks repetition and diversity against length-matched
references. At $T = 256$ the PAC-greedy mechanism produces text with rep-3 of 0.825 ---
but so does the public model (0.831) and the non-private fine-tuned model (0.836). The
mechanism is marginally \emph{less} repetitive than either baseline.
Loop absorption is a property of greedy decoding of GPT-2, not of privatization.
\begin{figure}[t]
\centering
\includegraphics[width=\linewidth]{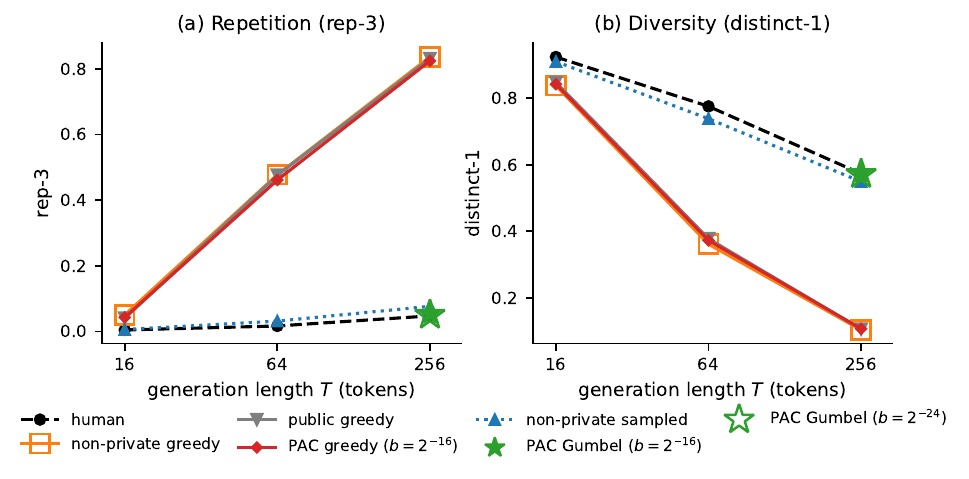}
\caption{Degeneracy is a property of the decoder, not the mechanism. Repetition (a) and
diversity (b) against generation length, all references generated at matched length:
PAC-greedy ($b=2^{-16}$) is indistinguishable from the public and non-private greedy
baselines, while PAC under coupled Gumbel decoding, shown at $T=256$ for
$b=2^{-16}$ and $2^{-24}$, lands on the human curve.}
\label{fig:degeneracy}
\end{figure}

Degeneracy interacts with the accounting in a way worth flagging for anyone
benchmarking private generation. Once the text cycles, all worlds agree, no noise is
drawn, and the measured privacy cost falls: the perplexity ratio against a length-matched
reference declines from 1.13--1.19 at $T = 16$ to 1.06--1.09 at $T = 256$ under greedy
decoding, while measured unanimity rises from 0.825 to 0.961. A low apparent privacy cost
at long generation lengths can be an artifact of degeneracy.

\paragraph{Coupling removes it at unchanged accounting.}
Under coupled Gumbel at $T = 256$, rep-3 falls to 0.053 and 0.050 at $b = 2^{-16}$ and
$2^{-24}$ --- within 0.006 of human text (0.047) --- and distinct-1 reaches 0.574 and
0.571 against human 0.573. The perplexity ratio against the length-matched sampled
reference is 1.11--1.16, the same range as at $T = 64$, so the measured cost no longer
grows with sequence length once degeneracy is removed. Section~\ref{sec:coupled-decoding}
establishes that this change leaves the privacy accounting untouched.

\paragraph{Which coupling.}
Table~\ref{tab:ablation} compares vote-formation rules in a paired design at $T = 64$,
with prompts, worlds and noise seeds held fixed. Temperature is a real knob:
$\tau = 0.8$ raises unanimity from 0.826--0.832 to 0.840--0.845 and cuts the measured
cost to $1.02\times$ against its own matched reference, but costs diversity (distinct-1
0.67 against 0.75--0.76 at $\tau = 1.0$). Under fixed-$b$ accounting the extra unanimity
buys no budget, so we use $\tau = 1.0$. The hybrid rule is rejected on a different
ground: it attains comparable unanimity but its dissent is far deeper --- mean 5.0--5.7
distinct votes with a 95th percentile of 28--36 --- so when it misses the unanimous
branch, the vote scatters. Agreement frequency and dissent depth are separate
quantities, and only the second is expensive under any entropy-sensitive accounting.

\begin{table}[t]
\centering
\caption{Coupling ablations, $T = 64$, paired design, 50 trials per cell, at
$b = 2^{-8}$ / $2^{-16}$. Perplexity ratios are against each variant's own
decoder-matched reference.}
\label{tab:ablation}
\begin{tabular}{lccccc}
\toprule
variant & unanimity & $a \mid$ dissent & p95 $a$ & PPL ratio & rep-3 / distinct-1 \\
\midrule
greedy              & 0.887 / 0.880 & 2.20 / 2.20 & 3 / 3   & 1.11 / 1.28 & 0.430 / 0.391 \\
Gumbel $\tau{=}1.0$ & 0.826 / 0.832 & 2.17 / 2.20 & 3 / 3   & 1.16 / 1.10 & 0.022 / 0.750 \\
Gumbel $\tau{=}0.8$ & 0.840 / 0.845 & 2.17 / 2.16 & 3 / 3   & 1.02 / 1.02 & 0.047 / 0.681 \\
hybrid ($\beta{\ge}0.9$) & 0.837 / 0.834 & 4.97 / 5.71 & 28 / 36 & 1.15 / 1.14 & 0.020 / 0.756 \\
\bottomrule
\end{tabular}
\end{table}

\subsection{Canary Stress Tests}
\label{sec:canary-results}

\paragraph{Memorization does not arise on its own.}
At the canary trigger, member and non-member worlds are statistically indistinguishable:
across all repetition counts up to 64, log-probability gaps lie within $\pm 0.12$, mean ranks agree to within two positions for C-COMMON
(87--157 for both) and are very similar for C-OOD
($\approx26,000$--$33,000$ for both), and no world ever votes for the canary token. C-OOD canaries are never in the
public support at all, so they cannot be released regardless of what the adapters learned.
A dedicated capacity pilot pushes further: at rank 32 and 4 epochs, raising repetitions
from 1{,}024 to 4{,}096 moves the log-probability gap from 0.33 to 1.14 and the median
rank from 62.5 to 52.5, but the fraction of positions at which the canary becomes the
argmax remains exactly 0. In this regime a canary must be injected to be studied.

\paragraph{Forced memorization.}
We therefore force perfect memorization into exactly the member worlds at inference time.
Table~\ref{tab:canary} shows the result under both decoders. Without noise, membership is
essentially perfectly recoverable, with an advantage of 0.999--1.000. With calibration,
the advantage collapses: by $b = 2^{-16}$ it is statistically indistinguishable from zero
under both decoders.

\begin{table}[t]
\centering
\caption{Forced-memorization membership advantage, $P(\text{emit} \mid \text{member})
- P(\text{emit} \mid \text{non-member})$, for two C-COMMON canaries, 2{,}000 trials per
cell. Standard error $\approx 0.022$, so the 95\% interval is roughly $\pm 0.044$;
``n.s.'' marks an advantage whose interval contains zero.}
\label{tab:canary}
\begin{tabular}{lcc}
\toprule
$b$ & greedy & coupled Gumbel \\
\midrule
no noise  & $+1.000$ / $+1.000$ & $+0.999$ / $+1.000$ \\
$2^{-4}$  & $+0.281$ / $+0.260$ & $+0.281$ / $+0.261$ \\
$2^{-8}$  & $+0.092$ / $+0.120$ & $+0.058$ / $+0.109$ \\
$2^{-16}$ & $-0.003$ / $+0.006$ (n.s.) & $-0.017$ / $-0.007$ (n.s.) \\
\bottomrule
\end{tabular}
\end{table}

The construction suggests why: in this controlled single-query setting, where the
posterior is still uniform, a record memorized by exactly $m/2$ worlds splits the vote
evenly, which is the maximum-variance configuration and therefore the one at which
calibration injects the most noise. The argument is specific to that setting and does not
extend to an arbitrary evolved posterior. The leakage at $b = 2^{-4}$ is real but remains
within its single-query bound, and that budget is unusable in any case, since it exhausts
the total budget after 16 tokens.

\paragraph{Inference protection is not content protection.}
At $b = 2^{-16}$ the two guarantees separate cleanly. Membership advantage is
statistically indistinguishable from zero under both decoders, yet the canary payload is
still emitted at close to the same rate under both membership hypotheses: 48.9--51.3\%
when the realized world is a member and 50.6--51.2\% when it is not. The mechanism makes
the release uninformative about membership; it does not prevent the string from
appearing. Public-support restriction does block off-support emission structurally ---
C-OOD canaries are never emitted at any budget --- but that is a statement about tokens
outside $V_t$, not about in-support sequence reproduction.

\subsection{Comparison with PMixED}
\label{sec:pmixed-results}

\paragraph{What is being compared.}
PMixED is the closest differentially private baseline to our setting: it also serves
next-token predictions from an ensemble fine-tuned on a private corpus, and it also mixes
those predictions with a public model. It differs from us in three ways. It requires each
record to influence exactly one expert, so its ensemble is a disjoint partition rather
than our overlapping worlds. It privatizes by moving each expert's output distribution
toward the public model and averaging, rather than by adding noise calibrated to
disagreement. And it subsamples experts per query, using the resulting amplification to
lower its accounted cost.

The two guarantees are also different in kind. PMixED reports
$(\varepsilon, \delta)$-differential privacy, which bounds the influence of any record in
the worst case over datasets; we bound mutual information under a stated input
distribution. DP is the stronger and more conservative statement. To compare them at all,
we convert both to the same operational quantity --- a provable upper bound on
membership-inference success --- and evaluate each method at the configuration that
achieves that bound.

\paragraph{Our setting versus theirs.}
PMixED's reported configuration uses $N = 80$ experts at subsampling rate $q = 0.03$,
with ablations to $N = 100$ \citep{flemings2024}. Applying those ensemble sizes to our corpus would leave very few articles per expert: disjoint sharding of a 2{,}048-article universe at $N = 80$ leaves about 26
articles per expert. We therefore restrict the shard count to $N \in \{8, 16, 32\}$ and
search over it at every operating point. The subsampling rate is not restricted; we sweep
it down to $10^{-4}$, well below the value PMixED reports, so that the accountant
is free to buy as much amplification as it wants. We note below which way the shard-count
restriction cuts.

\paragraph{Implementation validation.}
Before comparing, we check that our implementation reproduces PMixED's behaviour. We run
it unamplified ($q = 1.0$) at $\varepsilon_G = 2$, $T = 1024$, $\alpha = 2$, $E = 10$,
which is the strictest setting we evaluate and therefore the hardest place to show a
gain. Table~\ref{tab:pmixed-val} reports the result. Accuracy improves over the public
model at every shard count and the confidence interval excludes zero in each case, so the
private experts are contributing. The mean mixing weight $\lambda$ is 0.15--0.33, rising
to 0.75 once $\alpha$ and $q$ are tuned, confirming that the mollification step is active
rather than collapsing to the public model. No probability ever hits the numerical floor,
so the reported perplexities are real. Perplexity falls 5.6\% below public here and
11.9\% with $\alpha$ and $q$ tuned, against the 17.3\% the authors report on the full
corpus \citep{flemings2024}: the same direction at smaller magnitude, which is what a
2{,}048-article universe predicts.

\begin{table}[t]
\centering
\caption{PMixED validation, unamplified ($q = 1.0$) at $\varepsilon_G = 2$, $T = 1024$,
$\alpha = 2$, $E = 10$. Public baseline: accuracy 0.2438, perplexity 27.75. Headroom is
the share of that configuration's own non-private ceiling --- the plain average of all $N$
experts with no privatization --- that survives.}
\label{tab:pmixed-val}
\begin{tabular}{lccccc}
\toprule
$N$ & accuracy & perplexity & gain vs.\ public (95\% CI) & mean $\lambda$ & headroom \\
\midrule
8  & 0.2459 & 27.19 & $+0.0021$ $[+0.0019, +0.0024]$ & 0.151 & 3.7\% \\
16 & 0.2471 & 26.68 & $+0.0033$ $[+0.0029, +0.0037]$ & 0.221 & 6.9\% \\
32 & 0.2480 & 26.20 & $+0.0042$ $[+0.0036, +0.0048]$ & 0.334 & 11.8\% \\
\bottomrule
\end{tabular}
\end{table}

\paragraph{Two readings of PMixED's budget.}
The official PMixED implementation converts the target
$(\varepsilon,\delta)$-DP guarantee into the corresponding RDP budget before
applying its accountant. We use this implementation-consistent conversion for
the primary comparison. As a sensitivity analysis, we also report a more
generous convention that sets the internal RDP budget equal to the corresponding
DP $\varepsilon$, thereby granting PMixED a larger effective budget. We refer
to these as the rigorous and generous conventions, respectively. The generous
results are more favourable to PMixED and are quoted wherever a single number
must stand for PMixED.

\begin{figure}[t]
\centering
\includegraphics[width=\linewidth]{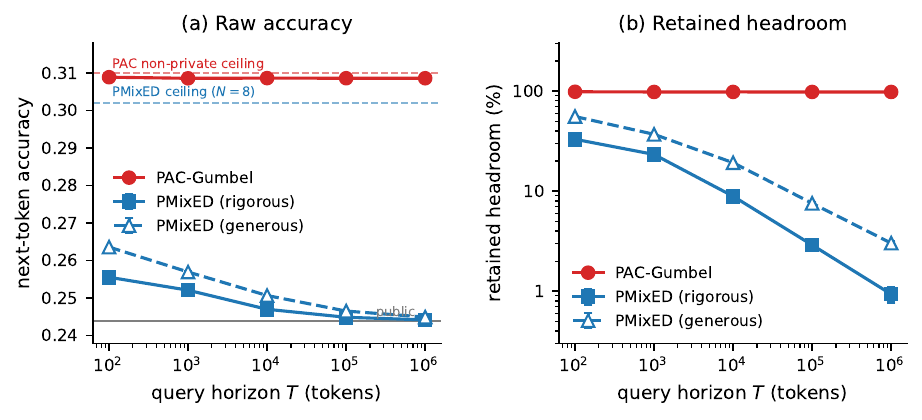}
\caption{PAC versus PMixED at a matched membership-inference bound
($\varepsilon \approx 1$), on the same data universe and held-out test set. At every
horizon, PMixED is evaluated over all combinations of shard count $N \in \{8, 16, 32\}$,
subsampling rate $q$ swept down to $10^{-4}$, and R\'enyi order $\alpha$, and the
best-performing configuration subject to the privacy target is the one plotted; the
selected configuration therefore varies along the curve, and is listed in
Table~\ref{tab:pmixed-full}.
\textbf{(a)} Raw next-token accuracy versus query horizon. Reference lines show the
public baseline, the PAC-Gumbel non-private ceiling (0.3099), and the highest
evaluated PMixED non-private ceiling (0.3021 at $N=8$). The PMixED configurations
selected by the privacy-constrained search use $N=32$, whose non-private ceiling
is 0.2793. PMixED approaches the public baseline as the horizon grows, whereas
PAC-Gumbel remains close to its non-private ceiling. \textbf{(b)} The same results normalized
by the non-private ceiling of the configuration selected at each point. Note the
logarithmic vertical scale. Error bars are standard errors over 16 subsampling draws per
cell and are smaller than the markers. Exact values and selected configurations are in
Table~\ref{tab:pmixed-full}.}
\label{fig:pmixed}
\end{figure}

Table~\ref{tab:pmixed-full} gives the numbers behind Figure~\ref{fig:pmixed}. The
matched operating point is a provable membership-inference bound of $0.7311$, which
corresponds to $\varepsilon \approx 1.0$ under the tight RDP-to-DP conversion and to a
total PAC budget of $0.111$ nats; the per-token budget $b$ is that total divided by the
horizon. For PMixED we report the configuration its own search selected at each point,
since the shard count $N$, the R\'enyi order $\alpha$ and the subsampling rate $q$ are all
free and are chosen to maximize accuracy subject to the privacy target. Headroom is
computed against the non-private ceiling of that same selected configuration; those
ceilings are $0.3021$, $0.2915$ and $0.2793$ at $N = 8$, $16$ and $32$ respectively,
against $0.3099$ for PAC-Gumbel and a public baseline of $0.24379$. The column
$f_\emptyset$ is the fraction of queries on which Poisson subsampling returned an empty
expert set, so that the query was answered by the public model alone; $\bar\lambda$ is the
mean mixing weight over the queries that were answered privately.

\begin{table*}[t]
\centering
\caption{Exact values for the matched comparison of Figure~\ref{fig:pmixed}
($\varepsilon \approx 1$). Accuracies are next-token accuracy on the held-out test
positions; headroom is the share of the selected configuration's own non-private ceiling
that survives privatization. PMixED columns give the configuration chosen by its own
search at that operating point. Uncertainties are standard errors over 16 subsampling
draws.}
\label{tab:pmixed-full}
\small
\setlength{\tabcolsep}{3pt}
\renewcommand{\arraystretch}{1.05}
\resizebox{\linewidth}{!}{%
\begin{tabular}{llcc ccccccc ccccccc}
\toprule
& & \multicolumn{2}{c}{PAC-Gumbel}
& \multicolumn{7}{c}{PMixED (rigorous)}
& \multicolumn{7}{c}{PMixED (generous)} \\
\cmidrule(lr){3-4}\cmidrule(lr){5-11}\cmidrule(lr){12-18}
$T$ & $b$ (nats) & acc. & head.\ (\%)
& acc. & $N$ & $\alpha$ & $q$ & $f_\emptyset$ & $\bar\lambda$ & head.\ (\%)
& acc. & $N$ & $\alpha$ & $q$ & $f_\emptyset$ & $\bar\lambda$ & head.\ (\%) \\
\midrule
$10^2$ & $1.11\times10^{-3}$ & 0.3089 & 98.4
& 0.25550 & 32 & 16 & 0.02   & 0.524 & 0.626 & $33.0 \pm 0.8$
& 0.26355 & 32 & 2  & 0.05   & 0.190 & 0.927 & $55.7 \pm 0.3$ \\
$10^3$ & $1.11\times10^{-4}$ & 0.3086 & 98.0
& 0.25205 & 32 & 32 & 0.01   & 0.728 & 0.743 & $23.3 \pm 0.7$
& 0.25691 & 32 & 2  & 0.03   & 0.373 & 0.887 & $37.0 \pm 0.5$ \\
$10^4$ & $1.11\times10^{-5}$ & 0.3086 & 98.1
& 0.24695 & 32 & 64 & 0.005  & 0.856 & 0.685 & $8.9 \pm 0.3$
& 0.25062 & 32 & 2  & 0.02   & 0.525 & 0.767 & $19.3 \pm 0.3$ \\
$10^5$ & $1.11\times10^{-6}$ & 0.3086 & 98.0
& 0.24482 & 32 & 64 & 0.001  & 0.971 & 0.970 & $2.9 \pm 0.3$
& 0.24648 & 32 & 2  & 0.01   & 0.724 & 0.649 & $7.6 \pm 0.2$ \\
$10^6$ & $1.11\times10^{-7}$ & 0.3086 & 98.0
& 0.24413 & 32 & 64 & 0.0005 & 0.984 & 0.773 & $0.9 \pm 0.2$
& 0.24487 & 32 & 2  & 0.005  & 0.850 & 0.483 & $3.0 \pm 0.1$ \\
\bottomrule
\end{tabular}%
}
\end{table*}

\paragraph{Result.}
Figure~\ref{fig:pmixed} shows the comparison. Panel (a) plots raw next-token accuracy
against the query horizon $T$, with three reference lines: the public model at 0.2438,
our non-private ceiling at 0.3099, and PMixED's most favourable non-private ceiling at
0.3021 ($N = 8$). PAC-Gumbel remains close to its non-private ceiling across the whole
range, at 0.3086--0.3089. PMixED falls steadily toward the public model, from 0.2636 at
$T = 10^2$ to 0.2449 at $T = 10^6$ under the generous reading, and from 0.2555 to 0.2441
under the rigorous one.

Panel (b) removes the effect of the two methods having different ceilings. Each point is
the fraction of that method's own non-private headroom that survives privatization,
$(A_{\text{private}} - A_{\text{public}}) / (A_{\text{non-private}} - A_{\text{public}})$,
so 100\% means no accuracy loss relative to that configuration's non-private ceiling and
0\% means accuracy equal to the public baseline. The denominator is always the ceiling of
the configuration actually selected at that point, which for PMixED means the $N$ its own
search chose there. PAC-Gumbel retains 98\% at every horizon from $10^2$ to $10^6$.
PMixED retains at most $55.7 \pm 0.3\%$, at its smallest horizon under the generous
reading; by $10^6$ it retains $3.0 \pm 0.1\%$ under that reading and $0.9 \pm 0.2\%$ under
the rigorous one. The margin is narrowest at $T = 10^2$, where $98.4\%$ still exceeds
$55.7\%$ by a factor of 1.8. There is no crossover at any horizon we evaluate, under
either reading. Each PMixED cell averages 16 subsampling draws with per-configuration
seeding, and the intervals above are standard errors across those draws.

At $\varepsilon = 2$, where PMixED has more budget to spend, the picture is unchanged in
shape: PAC-Gumbel retains 98.0--99.0\%, and PMixED retains at most $44.9 \pm 0.3\%$ under
the generous reading, falling to $4.1 \pm 0.2\%$ by $T = 10^6$.

\paragraph{What the accounting forces.}
The headroom figures understate what is happening to the method. As the horizon grows the
accountant drives the selected subsampling rate down --- from $q = 0.05$ at $T = 10^2$ to
$q = 5 \times 10^{-3}$ at $10^6$ under the generous reading, and from $0.02$ to
$5 \times 10^{-4}$ under the rigorous one --- and a small $q$ leaves the sampled expert
set empty on most queries, which then fall back entirely to the public model. Under the
generous reading the fallback rate rises from $19\%$ of queries at $T = 10^2$ to $53\%$ at
$10^4$ and $85\%$ at $10^6$; under the rigorous reading it is already $52\%$ at $T = 10^2$
and reaches $98\%$ by $10^6$. Under a tight budget the accounting does not merely shrink
the private contribution at each release; it drives the method to decline to consult the
private ensemble on most queries. The search accepts this because the amplification bought
by a small $q$ is worth more than the utility lost on the queries answered publicly ---
the same trade, in a sharper form, as its preference for the shard count with the lowest
ceiling.

\paragraph{Why.}
This is not a difference in base models. Panel (a) makes the point without any
normalization: both methods are built on the same public model, and PMixED's best
non-private ceiling, 0.3021 at $N = 8$, is close to our 0.3099, so both start about 0.06
above the public baseline. The mechanism of the decay is visible in the accounting.
PMixED buys privacy by projecting toward the public model, and under a fixed total budget
spread over $T$ queries, the per-query allowance shrinks as $T$ grows, so the accountant
must either shrink $\lambda$ or shrink $q$, and both reduce what the experts contribute.
Our mechanism behaves differently: calibration noise is required only at positions where
the worlds disagree, although the fixed accounting still charges $b$ for every privately
served token. A second observation points the same way: at $\varepsilon \approx 1$
PMixED's search selects $N = 32$ at every operating point even though $N = 32$ has the
\emph{lowest} ensemble ceiling of the three (0.2793, against 0.3021 at $N = 8$), meaning
the amplification from more shards is worth more to it than the quality it gives up.

\paragraph{Caveats.}
Our ensemble sizes are smaller than PMixED's reported $N = 80$, and its own search prefers
the largest $N$ we allow despite the utility cost. A larger corpus permitting $N \ge 64$
could change this trade-off in PMixED's favour, since larger ensembles buy more
amplification but give each expert less data; we do not evaluate that regime. The
subsampling rate is not restricted, so this is the only structural limit we impose.
Absolute numbers are not comparable to the published PMixED results, because our universe
is 2{,}048 articles rather than the full corpus; only the direction and the validated gain
transfer. Headroom is measured against each method's full non-private ensemble
($\lambda = 1$, $q = 1$), so the utility cost of subsampling counts as a privacy cost,
which is appropriate because $q$ is chosen by the privacy accountant. Perplexity is
reported for PMixED only, for the reason given in
Section~\ref{sec:utility-metrics}. And the comparison uses PAC-Gumbel throughout:
PAC-greedy reaches higher accuracy (0.4247 non-private ceiling) but is an argmax
interface, and comparing it to a sampling baseline would charge the difference between
decoders as a privacy result.

%% file: sections/07_discussion.tex
\section{Discussion and Limitations}
\label{sec:discussion}

\paragraph{Where the headroom is.}
Fixed per-token accounting charges $b$ at every release, while measured leakage is about
17\% of what is charged (Section~\ref{sec:privacy-results}). Most releases are unanimous
and draw no noise, and the dissent fraction is numerically close to the measured ratio,
which suggests the two are related --- but they are measured on different protocols, so
we treat the agreement as suggestive rather than as a demonstrated cause. Taken at face
value the gap indicates a factor of roughly $5.8$ in unexploited budget. Reclaiming any
of it safely would require an accounting rule that charges the realized vote distribution
instead of a constant, together with a composition argument for that rule which we have
not made; and such a rule would give an adversary who can predict the public coins a way
to influence the realized charge, a concern that does not arise under fixed $b$
(Remark~\ref{rem:coin-order}). It would also reframe the coupling design problem, since
the ablations show a rule can raise unanimity while making its dissent much deeper
(Section~\ref{sec:generation-results}), and dissent depth is what such a rule would
charge for.

\paragraph{Adaptive temperature.}
The coupling temperature $\tau$ shapes the cross-world vote distribution, hence $A_t$,
$C_t$, and $\Sigma_t$. A history-dependent rule $\tau_t = f(H_{t-1})$, or one keyed to
public quantities at the current context, is already admissible under
Section~\ref{sec:theory}: it is measurable with respect to information the adversary
holds and independent of the realized world. We have not evaluated one. We also do not
claim that moving $\tau$ in either direction monotonically reduces leakage, and the
limits suggest it does not --- as $\tau$ grows the shared coins dominate the per-world
logits, so the worlds agree and no noise is required, but the output stops reflecting the
private models at all.

\paragraph{What the guarantee does not cover.}
The accounting bounds inference about the secret, not reproduction of content. The canary
experiments make the distinction concrete: at $b = 2^{-16}$ membership advantage is
statistically indistinguishable from zero, while the memorized string is still emitted at
essentially the same rate under both membership hypotheses
(Section~\ref{sec:canary-results}). Restricting releases to the public support blocks
off-support emission structurally, but that is a claim about tokens outside $V_t$, not
about in-support reproduction. A deployment that needs content guarantees needs a
separate mechanism.

\paragraph{Limitations.}
The evaluation is a single corpus (WikiText-103), a single base model (GPT-2-small), a
2{,}048-article universe, and a single world collection. Absolute numbers should not be
read against those in the original PMixED paper, since model, data, and training scale
all differ; our head-to-head uses our own reimplementation of PMixED under the matched
setup of Section~\ref{sec:baselines}, validated against the authors' default
configuration before use. We did not train a second world collection, so we cannot
separate properties of the $m/2$ construction from properties of this particular
partition. We do not directly evaluate AdaPMixED, whose data-dependent accounting
differs from the fixed-guarantee comparison considered here. Three
further gaps are worth naming. The
accounting is charged on the internal history while only tokens are released, and we do
not quantify that slack (Remark~\ref{rem:conservatism}). The empirical attacks read the
curator's own posterior, so they are an upper envelope on a token-only adversary rather
than a realistic attacker, and a weaker measured attack would not be evidence of a
stronger guarantee. And the direct mutual-information result estimates $I(S; H_T)$ from
finitely many sessions: the identity is exact, the expectation is not.

\paragraph{Future work.}
The most valuable additions, in order: a second corpus and domain, which would answer the
single-dataset objection most directly; the entropy-charging odometer, which the $5.8\times$
gap motivates; and distillation from PAC-private tokens, which would convert a finite
query budget into a publishable model.

%% file: sections/08_conclusion.tex
\section{Conclusion}
\label{sec:conclusion}

We adapted PAC-private prediction to autoregressive generation by making the secret space
small enough to evaluate at serving time: 128 overlapping worlds, one adapter each, with
noise calibrated per token to how much the worlds disagree. The adaptation required
establishing that a token interface, a self-generated query stream, a singular vote
covariance, and a coupled sampling decoder all remain inside the posterior-aware
composition framework, which Section~\ref{sec:theory} does.

Empirically, the mechanism retains at least 74\% of the fine-tuning gain at a per-token budget
that bounds membership inference below 51.08\% after $10^6$ released tokens, and accuracy
is flat across a $2^{28}$ range of budgets --- not because the budget is inert, but
because most positions are unanimous and therefore draw no noise, while the additional
flips induced on dissenting positions are close to accuracy-neutral. Because the $m$-world
construction keeps the secret space enumerable, we could also estimate the leakage
directly from posterior entropy rather than probing it with an attack, and found it to be
roughly 17\% of what the fixed accounting charges. Against PMixED at matched
membership-inference bounds on the same data, PAC retained 98\% of its non-private
headroom at every query volume tested, with no crossover.

Two broader lessons emerge from our experiments. First, privacy leakage can appear
artificially small when the decoder produces highly repetitive or degenerate text, because
such outputs are also those on which the ensemble tends to agree. Second, low membership
inference does not imply that memorized content cannot still be reproduced: in our canary
experiment, the memorized string remained emittable even when the measured membership
advantage was statistically indistinguishable from zero. These observations are not specific
to our mechanism and should be considered when evaluating private generative models more
generally.

%% file: appendix/appendix.tex
\section{Additional Proofs}
\label{app:proofs}

\subsection{Proof of Theorem~\ref{thm:main}}
\label{app:thm-main}

By the chain rule applied to $H_T = \{(M_t^{U_t}, U_t, R_t)\}_{t=1}^T$,
\[
  I(S; H_T) \;=\; \sum_{t=1}^{T}
  \Big[ I\big(S; M_t^{U_t}, U_t \mid H_{t-1}\big)
      + I\big(S; R_t \mid H_{t-1}, M_t^{U_t}, U_t\big) \Big].
\]
For the first term, (C1) gives $I(S; U_t \mid H_{t-1}) = 0$, and
Proposition~\ref{prop:autoregressive} gives $M_t \perp S \mid H_{t-1}$; since
$M_t^{U_t}$ is a function of $(q_t, U_t)$ and the public adapters, the chain rule yields
$I(S; M_t^{U_t}, U_t \mid H_{t-1}) = 0$. For the second, (C2) ensures $M_t^{U_t}$ is a
deterministic map of $s$ at calibration time, so Proposition~\ref{prop:coupled}(i)
applies and bounds it by $b$; its calibration is valid despite $\operatorname{rank}
C_t \le a_t - 1$ by Lemma~\ref{lem:singular}(i). Summing over $t$ gives
$I(S; H_T) \le bT$.

For the first inequality of the theorem, Lemma~\ref{lem:token}(ii) gives
$I(S; Y_{1:T}) \le I(S; H_T)$. The attack-success bounds follow by the PAC privacy
conversion applied at total budget $B_T = bT$ with the relevant prior: $1/2$ for record
membership, since the $m/2$ construction places every record in exactly half the worlds,
and $1/m$ for identification of the realized world under the uniform prior. \qed

\subsection{Proof of Lemma~\ref{lem:singular}}
\label{app:singular}

Write $A_t = \{v_1, \dots, v_{a_t}\}$ and let $\Delta = \{x \in \mathbb{R}^{a_t} :
\mathbf{1}^\top x = 1\}$. Every realization of $M_t^{U_t}$ is a vertex $e_{v}$ of
$\Delta$, so all differences $e_v - e_{v'}$ lie in $\mathbf{1}^\perp$. Moreover
$C_t \mathbf{1} = \pi_t - \pi_t(\mathbf{1}^\top \pi_t) = 0$, so
$\operatorname{span}(\mathbf{1}) \subseteq \ker C_t$ and
$\mathcal{A} = \operatorname{range}(C_t) \subseteq \mathbf{1}^\perp$, hence
$\operatorname{rank} C_t \le a_t - 1$.
Let
\[
\mu_t=\sum_s P_{t-1}(s)y_s.
\]
Then
\[
C_t
=
\sum_s P_{t-1}(s)
(y_s-\mu_t)(y_s-\mu_t)^\top.
\]
Hence, for any vector $w$,
\[
w^\top C_t w
=
\sum_s P_{t-1}(s)
\bigl((y_s-\mu_t)^\top w\bigr)^2.
\]
Since $C_t$ is positive semidefinite, $w\in\ker(C_t)$ if and only if
$(y_s-\mu_t)^\top w=0$ for every $s$ with $P_{t-1}(s)>0$.
Therefore,
\[
\operatorname{range}(C_t)
=
\operatorname{span}
\{y_s-\mu_t : P_{t-1}(s)>0\}.
\]
(i) Decompose $M_t^{U_t}(
S) = \Pi_{\mathcal{A}} M_t^{U_t}(S) +
\Pi_{\mathcal{A}^\perp} M_t^{U_t}(S)$. By the identity above, for any $s,s'$ in the posterior support,
\[
y_s-y_{s'}
=
(y_s-\mu_t)-(y_{s'}-\mu_t)
\in\mathcal A.
\]
Therefore
\[
\Pi_{\mathcal A^\perp}y_s
=
\Pi_{\mathcal A^\perp}y_{s'},
\]
so $\Pi_{\mathcal A^\perp}M_t^{U_t}(S)$ is constant in $S$.
A component that is constant in $S$ is independent of $S$ and contributes no mutual
information, so
$I(S; M_t^{U_t}(S) + Z_t) = I(S; \Pi_{\mathcal{A}}M_t^{U_t}(S) + \Pi_{\mathcal{A}}Z_t)$,
which is at most $b$ by the noise determination result applied within $\mathcal{A}$,
where $\Sigma_t$ is by construction the calibration it prescribes.

(ii) Conditional on $\mathcal{A}$, the law of $R_t$ given $S = s$ is
$\mathcal{N}(\Pi_{\mathcal{A}} y_s, \Sigma_t)$ restricted to $\mathcal{A}$, a
nondegenerate Gaussian on that subspace with density proportional to
$\exp[-\tfrac12 (R_t - y_s)^\top \Sigma_t^{+} (R_t - y_s)]$. The normalizing constant
involves the pseudo-determinant of $\Sigma_t$, which does not depend on $s$ and
therefore cancels when $P_t$ is normalized over $\mathcal{S}$. Bayes' rule then gives
the stated update.

(iii) If $a_t = 1$ then $\pi_t$ is a point mass, $C_t = 0$, $\mathcal{A} = \{0\}$, and
$\Sigma_t = 0$. The output is constant across worlds, so $I(S; M_t^{U_t}(S)) = 0$ and
the step is free. \qed

\subsection{Proof of Proposition~\ref{prop:coupled}}
\label{app:coupled}

(i) Condition on $(H_{t-1}, M_t^{U_t}, U_t)$. By (C2) the coins are realized before
calibration, so on this conditioning event $s \mapsto M_t^{U_t}(s)$ is a fixed
deterministic map, which is the hypothesis under which the noise determination result
holds. Its output covariance under $P_{t-1}$ is exactly $C_t$, the matrix from which
$\Sigma_t$ was calibrated to budget $b$; Lemma~\ref{lem:singular}(i) then gives
$I(S; R_t \mid H_{t-1}, M_t^{U_t}, U_t) \le b$.

(ii) By (C1), $I(S; U_t \mid H_{t-1}) = 0$. Given $(H_{t-1}, U_t)$, the map $M_t^{U_t}$
is determined by $q_t$ and the public adapters, and $q_t$ depends on $S$ only through
$H_{t-1}$ by Proposition~\ref{prop:autoregressive}; hence
$I(S; M_t^{U_t} \mid H_{t-1}, U_t) = 0$. The chain rule gives
$I(S; M_t^{U_t}, U_t \mid H_{t-1}) = 0$.

(iii) The update multiplies $P_{t-1}$ by the exact conditional density of $R_t$ given
$S = s$ and $(H_{t-1}, M_t^{U_t}, U_t)$, computed in
Lemma~\ref{lem:singular}(ii). By (C3) all three conditioning variables are components of
$H_t$, so the normalized product is $P_{S \mid H_t}$.

Note that (i)--(iii) nowhere use independence between $U_t$ and $q_t$; only (C1)--(C3)
are used. \qed

\subsection{Coins: verification and scope}
\label{app:coins}

\paragraph{Context-derived coins satisfy (C1)--(C3).}
Let $U_t = f(q_t, H_{t-1})$ for a public measurable $f$. Then $U_t$ is
$\sigma(H_{t-1}, M_t)$-measurable. Since $M_t \perp S \mid H_{t-1}$
(Proposition~\ref{prop:autoregressive}), the data-processing inequality for conditional
mutual information gives $I(S; U_t \mid H_{t-1}) \le I(S; M_t \mid H_{t-1}) = 0$, so
(C1) holds. $U_t$ is determined as soon as $(H_{t-1}, q_t)$ are, hence before
calibration, so (C2) holds. Since $f$ is public and $(H_{t-1}, q_t)$ are recorded, $U_t$
is recoverable from the history, so (C3) holds. Consequently
Proposition~\ref{prop:coupled} and Theorem~\ref{thm:main} apply verbatim to such a
scheme.

\paragraph{Adversarial coin prediction.}
Suppose the adversary can compute $U_t$ in advance and searches over prompts for one
whose coins induce a vote pattern it considers favourable. The per-step bound
$I(S; R_t \mid H_{t-1}, M_t^{U_t}, U_t) \le b$ is conditional on the realized pair and
holds for every realization, so no such search can exceed the per-step charge. This is
the operational content of instance-based calibration: the noise is matched to whatever
mechanism is submitted, rather than to a worst case fixed in advance.

\paragraph{What (C1) excludes.}
Coins seeded from the realized world index, from the realized world's logits at step
$t$, or from any quantity depending on $S$ other than through $H_{t-1}$ violate (C1) and
are not covered. The prohibition is on dependence on the secret, not on dependence on
the query.

\paragraph{Caveat for entropy-charging variants.}
Under a variant of the accounting in which the per-step charge depends on the realized
vote distribution rather than being fixed at $b$ --- for instance charging the vote
entropy --- an adversary able to predict the coins could steer the realized charge, and
the argument above would require revisiting. All results in this paper use the fixed
per-token budget $b$, for which the issue does not arise.

\subsection{Why the bound cannot be routed through $R_{1:T}$ alone}
\label{app:not-R}

Under coupled decoding, $Y_t = \arg\max_{v \in A_t}(R_t)_v$ with $A_t$ determined by
$(q_t, U_t)$. Thus $Y_t$ is not a function of $R_t$ alone and $S \to R_{1:T} \to Y_{1:T}$
is not a Markov chain, so the data-processing inequality does not apply in that form.
Conditioning restores the deterministic map but not the inequality: one obtains
\[
  I(S; Y_{1:T}) \;\le\; I(S; Y_{1:T} \mid U_{1:T}) \;\le\; I(S; R_{1:T} \mid U_{1:T}),
\]
using $I(S; U_{1:T}) = 0$ for the first step and conditional data processing for the
second, and a conditional mutual information need not be bounded by its unconditional
counterpart. The joint history is the correct object: since $I(S; U_{1:T}) = 0$,
\[
  I(S; R_{1:T}, U_{1:T}) \;=\; I(S; U_{1:T}) + I(S; R_{1:T} \mid U_{1:T})
  \;=\; I(S; R_{1:T} \mid U_{1:T}),
\]
and it is this quantity, a component of $I(S; H_T)$, that the composition argument
bounds. Under greedy decoding $U_t$ is null, $A_t$ is $\sigma(H_{t-1}, M_t)$-measurable,
and the two routes coincide.


\subsection{Matched-Comparison Values}
\label{app:pmixed-values}

Table~\ref{tab:pmixed-full} reports the main $\varepsilon \approx 1$ comparison in
Section~\ref{sec:pmixed-results}. As a sensitivity check, Table~\ref{tab:pmixed-full-eps2}
repeats the comparison at $\varepsilon=2$, giving PMixED a looser privacy target while
keeping the same evaluation protocol and search space.
\begin{table*}[t]
\centering
\caption{The same comparison at the looser operating point, a membership-inference bound
of $0.8808$ ($\varepsilon = 2.0$, total PAC budget $0.328$ nats). Columns are as in
Table~\ref{tab:pmixed-full}.}
\label{tab:pmixed-full-eps2}
\small
\setlength{\tabcolsep}{3pt}
\renewcommand{\arraystretch}{1.05}
\resizebox{\linewidth}{!}{%
\begin{tabular}{llcc ccccccc ccccccc}
\toprule
& & \multicolumn{2}{c}{PAC-Gumbel}
& \multicolumn{7}{c}{PMixED (rigorous)}
& \multicolumn{7}{c}{PMixED (generous)} \\
\cmidrule(lr){3-4}\cmidrule(lr){5-11}\cmidrule(lr){12-18}
$T$ & $b$ (nats) & acc. & head.\ (\%)
& acc. & $N$ & $\alpha$ & $q$ & $f_\emptyset$ & $\bar\lambda$ & head.\ (\%)
& acc. & $N$ & $\alpha$ & $q$ & $f_\emptyset$ & $\bar\lambda$ & head.\ (\%) \\
\midrule
$10^2$ & $3.28\times10^{-3}$ & 0.3092 & 99.0
& 0.25668 & 32 & 8  & 0.03   & 0.375 & 0.649 & $36.3 \pm 0.8$
& 0.26520 & 16 & 2  & 0.1    & 0.185 & 0.853 & $44.9 \pm 0.3$ \\
$10^3$ & $3.28\times10^{-4}$ & 0.3087 & 98.2
& 0.25366 & 32 & 64 & 0.02   & 0.527 & 0.514 & $27.8 \pm 0.7$
& 0.25879 & 32 & 2  & 0.03   & 0.373 & 0.917 & $42.3 \pm 0.5$ \\
$10^4$ & $3.28\times10^{-5}$ & 0.3086 & 98.0
& 0.24845 & 32 & 32 & 0.005  & 0.848 & 0.870 & $13.1 \pm 0.6$
& 0.25233 & 32 & 2  & 0.02   & 0.525 & 0.840 & $24.1 \pm 0.3$ \\
$10^5$ & $3.28\times10^{-6}$ & 0.3086 & 98.1
& 0.24559 & 32 & 64 & 0.002  & 0.939 & 0.880 & $5.1 \pm 0.4$
& 0.24730 & 32 & 2  & 0.01   & 0.724 & 0.762 & $9.9 \pm 0.3$ \\
$10^6$ & $3.28\times10^{-7}$ & 0.3086 & 98.0
& 0.24433 & 32 & 32 & 0.0005 & 0.985 & 0.984 & $1.5 \pm 0.3$
& 0.24526 & 32 & 2  & 0.005  & 0.850 & 0.620 & $4.1 \pm 0.2$ \\
\bottomrule
\end{tabular}%
}
\end{table*}

Across Tables~\ref{tab:pmixed-full} and~\ref{tab:pmixed-full-eps2}, three features are worth noting. First, PMixED's search selects the largest
shard count we allow, $N = 32$, at every operating point but one, even though $N = 32$ has
the lowest ensemble ceiling of the three; the amplification from more shards outweighs the
quality lost by splitting the corpus further. The single exception is the generous reading
at $\varepsilon = 2$, $T = 10^2$, which selects $N = 16$ --- the one point in either table
where the budget is loose enough and the horizon short enough for ensemble quality to win.
Second, the selected subsampling rate falls monotonically with the horizon, reaching
$5 \times 10^{-4}$ under the rigorous convention and $5 \times 10^{-3}$ under the
generous convention at $T = 10^6$, correspondingly, the fraction of queries answered
by the public model alone rises to 98\% and 85\%, respectively. Third, the two
readings select systematically different R\'enyi orders: the generous reading always takes
$\alpha = 2$, while the rigorous reading is pushed to $\alpha \in \{8, \dots, 64\}$ to make
the tight conversion pay.






